\documentclass[10pt,twocolumn,letterpaper]{article}

\usepackage{iccv}
\usepackage{times}
\usepackage{epsfig}
\usepackage{color}
\usepackage{bm}
\usepackage{amsfonts,amsmath,amsthm}
\usepackage{graphicx}
\usepackage{subcaption}
\usepackage{booktabs}
\usepackage{enumitem}
\usepackage{float} 
\usepackage[ruled,vlined]{algorithm2e} 
\usepackage[pagebackref=true,breaklinks=true,colorlinks,bookmarks=false]{hyperref}
\usepackage{setspace}
\usepackage[capitalize]{cleveref}

\iccvfinalcopy

\def\iccvPaperID{3293} 

\ificcvfinal\fi

\title{Chordal Averaging on Flag Manifolds and Its Applications}

\author{Nathan Mankovich\\
Colorado State University
\and
Tolga Birdal\\
Imperial College London
}

\newcommand{\R}{\ensuremath{\mathbb{R}}}  
\newcommand{\Gr}{\ensuremath{\mathrm{Gr}}}  

\newcommand{\tr}{{\mathrm{tr}}}

\newcommand{\Y}{\mathbf{Y}}
\newcommand{\X}{\mathbf{X}}

\newcommand{\x}{\mathbf{x}}
\newcommand{\tb}{\mathbf{t}}

\newcommand{\U}{\mathbf{U}}
\newcommand{\M}{\mathbf{M}}

\newcommand{\V}{\mathbf{V}}

\newcommand{\z}{\mathbf{z}}

\newcommand{\m}{\mathbf{m}}
\newcommand{\I}{\mathbf{I}}

\newcommand{\Z}{\mathbf{Z}}

\newcommand{\Rot}{\mathbf{R}}
\newcommand{\zero}{\mathbf{0}}

\newcommand{\flag}{\mathcal{FL}}

\newcommand{\Id}{\mathbf{I}}

\newtheorem{remark}{Remark}

\newtheorem{prop}{Proposition}
\newtheorem{dfn}{Definition}

\renewcommand{\paragraph}[1]{{\vspace{1mm}\noindent \bf #1}.}

\DeclareMathOperator*{\argmin}{arg\min}

\crefname{eq}{eq}{eq}
\Crefname{Eq}{Eq}{Eq}
\crefname{thm}{theorem}{theorem}
\Crefname{Thm}{Theorem}{Theorem}
\crefname{prop}{Prop.}{Prop.}
\crefname{dfn}{Dfn.}{Dfn.}
\Crefname{Prop}{Proposition}{Proposition}
\crefname{remark}{remark}{remark}
\Crefname{Remark}{Remark}{Remark}
\Crefname{algorithm}{Alg.}{Alg.}

\begin{document}

\maketitle
\ificcvfinal\thispagestyle{empty}\fi

\begin{abstract}
This paper presents a new, provably-convergent algorithm for computing the flag-mean and flag-median of a set of points on a flag manifold under the chordal metric. The flag manifold is a mathematical space consisting of flags, which are sequences of nested subspaces of a vector space that increase in dimension. The flag manifold is a superset of a wide range of known matrix spaces, including Stiefel and Grassmanians, making it a general object that is useful in a wide variety computer vision problems.

To tackle the challenge of computing first order flag statistics, we first transform the problem into one that involves auxiliary variables constrained to the Stiefel manifold. The Stiefel manifold is a space of orthogonal frames, and leveraging the numerical stability and efficiency of Stiefel-manifold optimization enables us to compute the flag-mean effectively. Through a series of experiments, we show the competence of our method in Grassmann and rotation averaging, as well as principal component analysis.
We release our source code under \href{https://github.com/nmank/FlagAveraging}{https://github.com/nmank/FlagAveraging}.
\end{abstract}

\section{Introduction}\label{sec:intro}
Subspace analysis is a key workhorse of machine learning since various forms of data and parameter sets admit a compact representation as a subspace of a high-dimensional vector space. Diffusion imaging data~\cite{fletcher2007riemannian} or 
appearance variations of objects (\eg human faces) under variable lighting can be effectively modeled by low dimensional linear spaces~\cite{beveridge2008principal}, while a video as a whole can be modeled as the subspace that spans the observed frames~\cite{marrinan2014finding}. 

A large body of the aforementioned approaches leverage the mathematical framework of Grassmanian manifolds thanks to the ease in dealing with the confounding variability in observations~\cite{harandi2011graph,he2012incremental,hong2014geodesic,kumar2016robust}.  As such, they rely on statistical analysis tools inherently requiring mean or variance estimations on matrix manifolds~\cite{chakraborty2017intrinsic,chakraborty2015recursive,marrinan2014finding}. Yet, (i) they have been found to be susceptible to outliers, and (ii) while Grassmanians were suitable for analyzing \emph{tall} data where the ambient dimension is much larger than the number of data points, they become less effective when it comes to \emph{wide} data where the data dimension is relatively small~\cite{ma2021flag}. In such cases, the more structured \emph{flag manifolds} have been found to be more effective~\cite{ma2021flag}.

\begin{figure}[t]
        \includegraphics[width=\columnwidth]{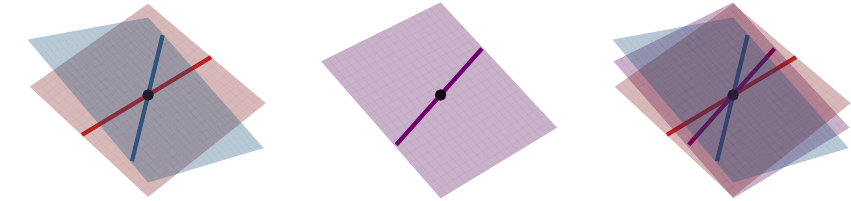}\vspace{-2mm}
	\caption{\textbf{Chordal averaging} on the \textbf{flag manifold} $\flag(1,2;3)$. The average (shown in purple) of the input (red and blue) lines remain in the average of the input planes.\vspace{-4mm}}
	\label{fig:flagavg}
\end{figure}

A flag manifold is a nested series of subspaces geometrically generalizing Grassmanians. Any multilevel, multiresolution, or multiscale phenomena is likely to involve flags, whether implicitly or explicitly. This makes flag manifolds instrumental in dimensionality reduction, clustering, learning deep feature embeddings, visual domain adaptation, deep neural network compression and dataset analysis~\cite{minnehan2019deep,ma2021flag,zhang2018grassmannian}. Thus, computing statistics on flag manifolds becomes an essential prerequisite powering several downstream applications. 
In this paper, we propose an approach for computing first order statistics on (\emph{oriented}) flag manifolds (\cf~Fig. \ref{fig:flagavg})\footnote{While our averages are for general flag-manifolds, we do provide oriented averages for flag manifolds of type $1,2,3,\dots,d-1$ in $d$-D space.}. In particular, endowing flag manifolds with the non-canonical chordal metric, we first transform the (weighted) \emph{flag-mean} problem into an equivalent minimization on the Stiefel manifold, the space of orthonormal frames, via the method of Lagrange multipliers. We then leverage Riemannian Trust-Region (RTR) optimizers~\cite{boumal2014manopt,birdal2019probabilistic} to obtain the solution. Subsequently, we introduce an iteratively reweighted least squares (IRLS) scheme to estimate the more robust \emph{flag-median} as an $L_1$ flag-mean. Finally, we show how several common problems in computer vision such as motion averaging, can be translated onto averages on flag manifolds using group contraction operators~\cite{ozyesil2018synchronization}. In particular, our contributions are:
\begin{itemize}[itemsep=0.5pt,leftmargin=*,topsep=1pt]
    \item We introduce a new algorithm for computing flag-prototypes (\eg flag-mean and -median) of a set of points lying on the flag-manifold.
    \item Analogous to our flag-mean, we introduce an IRLS minimization to estimate the flag-median.
    \item We prove the convergence of the proposed IRLS algorithm for the flag-median.
    \item We show how rigid motions can be embedded into flags and thus provide a new way to robustly average motions.
\end{itemize}
 Our diverse experiments reveal that flag averages are more robust, usually yield more reliable estimates, and are more general, \ie, generalize Grassmannian averages. We will release our implementations upon publication.

 


\section{Related Work}\label{sec:related}


\paragraph{Flag manifolds}
Besides being mathematically interesting objects~\cite{wiggerman1998fundamental,donagi2008glsms,alekseevsky1997flag}, flags and flag manifolds have been explored by a series of works from Nishimori~\etal addressing subspace independent component analysis (ICA) via Riemannian optimization~\cite{nishimori2006riemannian0, nishimori2006riemannian1, nishimori2007flag, nishimori2008natural,nishimori2007flag,nishimori2006riemannian}. Nested sequences of subspaces (e.g. flags) appear in 
the weights in principal component analysis (PCA) \cite{ye2022optimization} and the result of a wavelet transform \cite{kirby2001geometric}.


\paragraph{Flag manifolds in computer vision} The utilization of flag manifolds in computer vision is a recent development. Ma~\etal~\cite{ma2021flag} employ nested subspace methods to compare large datasets. Additionally, they port self-organizing mappings to work on flag manifolds, enabling parameterization of a set of flags of a fixed type. This method was applied to hyper-spectral image data analysis~\cite{ma2022self}. Ye~\etal~\cite{ye2022optimization} derive closed-form analytic expressions for the set of operators required for Riemannian optimization algorithms on the flag manifold, while Nguyen~\cite{nguyen2022closed} provides algorithms for logarithmic maps and geodesics on flag manifolds. Marrinan~\etal~\cite{marrinan2014finding} investigate the averaging of Grassmanians into flags, demonstrating that flag means behave more like medians and are therefore more robust to the presence of outliers among the subspaces being averaged. Building on this work, they utilize flag averages to improve the detection of chemical plumes in hyperspectral videos~\cite{marrinan2016flag}. Finally, Mankovich~\etal~\cite{mankovich2022flag} also average Grassmannians into flags by providing the median as a flag and an algorithm to compute it. 


\vspace{-1mm}\section{Chordal Centroids on Flag Manifolds}\vspace{-1mm}\label{sec:method}
We begin by providing the necessary definitions related to flag manifolds before presenting our chordal flag-mean and -median algorithms. 


\begin{dfn}[Matrix groups] The \textbf{orthogonal group} $O(d)$ denotes the group of distance-preserving transformations of a Euclidean space of dimension $d$. $SO(d)$ is the \textbf{special orthogonal group} containing matrices in $O(d)$ determinant $1$. The \textbf{Stiefel manifold} $St(k,d)$, a.k.a. the set of all orthonormal $k$-frames in $\R^d$, can be represented as the quotient group: $St(k,d) = O(d)/O(d-k)$. A point on the Stiefel manifold is parameterized by a tall-skinny $d \times k$ real matrix with orthonormal columns.
The \textbf{Grassmannian}, $Gr(k, d)$, represents the collection of points parameterizing the $k$-dimensional subspaces of a fixed $d$-dimensional vector space, \eg $\R^d$. For our purposes, $Gr(k, d)$ is a \emph{real matrix manifold}, where each point is identified with an \emph{equivalence class of orthogonal matrices}, \ie $Gr(k, d)=O(d)/O(k)\times O(d-k)$. \\
\textbf{Notation:} We represent $[\X] \in Gr(k,d)$ using the truncated orthogonal matrix $\X \in \R^{d \times k}$. For this paper $[\X]$ is used to denote the subspace spanned by the columns of $\X$.
\end{dfn}

\begin{dfn}[Flag]\label{def:flagobj}
A \emph{flag} in a finite dimensional vector space $\mathcal{V}$ over a field is a sequence of nested subspaces with \emph{increasing} dimension, each containing its predecessor, \ie the filtration: $\{\emptyset \}=\mathcal{V}_0 \subset \mathcal{V}_1 \subset \dots \subset \mathcal{V}_{k} \subset \mathcal{V}$ with $0=d_0<d_1<\dots<d_k<d_{k+1} = d$ where $\mathrm{dim} \mathcal{V}_i=d_i$ and $\mathrm{dim} \mathcal{V}=d$. We say this flag is of \emph{type} or \emph{signature} $(d_1, \dots ,d_k,d)$. A flag is called \emph{complete} if $d_i = i,\,\forall i$. Otherwise the flag is \emph{incomplete} or \emph{partial}. 

\noindent\textbf{\emph{Notation}:} A flag, $[\![\X]\!]$ of type $(d_1, \dots ,d_k, d)$, is represented by a truncated orthogonal matrix $\X \in \R^{d \times d_k}$. Let $m_j = d_j-d_{j-1}$ for $j=1,2,\dots, k+1$, and $\X_j  \in \R^{d \times m_j}$ for $j=1,2,\dots, k$ whose columns are the $d_{j-1}+1$ to $d_{j}$ columns of $\X$. $[\![\X]\!]$ is
\begin{equation}
     [ \X_1 ]  \subset [\X_1, \X_2] \subset \cdots \subset  [\X_1, \dots, \X_k] = [\X] \subset \R^d. \nonumber
\end{equation}
\end{dfn}

\begin{dfn}[Flag manifold]
The aggregate of all flags of the same type, \ie a certain collection of ordered sets of vector subspaces, admit the structure of manifolds. We refer to this \emph{flag manifold} 
as $\flag(d_1, ..., d_k; d)$ or equivalently as $\flag(d+1)$\footnote{Note that we will use $\flag(d_1, ..., d_k; d)$ and $\flag(d+1)$ interchangeably in the rest of the manuscript.}. The points of $\flag(d+1)$
parameterize all flags of type $(d_1, ..., d_k,d)$. 
Flag manifolds generalize Grassmannians because $\flag(k; d)=\Gr(k,d)$. $\flag(d+1)$ can be thought of as a quotient of groups ~\cite{ma2022self}:
\begin{equation*}
\flag(d+1) = SO(d)/S(O(m_1) \times O(m_2) \times \cdots \times O(m_{k+1})).
\end{equation*}
\end{dfn}

\begin{dfn}[Chordal distance on the flag manifold~\cite{pitaval2013flag}]\label{def:flagdist}
For $[\![\X]\!], [\![\Y]\!] \in \flag(d+1)$, the \emph{chordal distance} is a map $d_c: \flag(d+1) \times \flag(d+1) \to \R$:
\begin{equation}\label{eq: chordal distance}
    d_c([\![\X]\!], [\![\Y]\!]) := \sqrt{\sum_{j=1}^k m_j - \tr(\X_j^{\top} \Y_j \Y_j^{\top} \X_j)}.
\end{equation}
\end{dfn}
We now endow flags with \emph{orientation}, which is required in certain applications such as motion averaging. 
\begin{dfn}[Oriented flag manifold~\cite{selig2005study, ma2022self}]
An \emph{oriented flag manifold}, $\flag^+(d+1)$, contains only flags with subspaces with compatible orientations. Algebraically:
\begin{equation*}
    \flag^+(d+1) = SO(d)/(SO(m_1) \times \cdots \times SO(m_{k+1})). 
\end{equation*}
Two oriented vector spaces have the same orientation if the determinant of the unique linear transformation between them is positive~\cite{alberti2005geometric}.
\end{dfn}

\subsection{The Chordal Flag-mean}
Armed with notation for flags (Dfn. \ref{def:flagobj}) and ways to measure distance between them (Dfn. \ref{def:flagdist}), we are prepared to state the chordal flag-mean estimation problem formally.
\begin{dfn}[Weighted chordal flag-mean]
Let $\{ [\![\X^{(i)}]\!] \}_{i=1}^p \subseteq \flag(d+1)$ be a set of points on a flag manifold with weights $\{\alpha_i\}_{i=1}^p \subset \R$ where $\alpha_i \geq 0$. The chordal flag-mean $[\![\bm{\mu}]\!]$ of these points solves:
\begin{equation} \label{eq: chordal flag mean opt}
    \argmin_{[\![\Y]\!]\in \flag(d+1)} \sum_{i=1}^p \alpha_i d_c([\![\X^{(i)}]\!], [\![\Y]\!])^2.
\end{equation}
\end{dfn}
\textit{Note: for $\flag(k;n)$, this amounts to the Grassmannian-mean by Draper~\etal~\cite{draper2014flag}.}
\begin{prop}\label{prop: flag mean}
The chordal flag-mean optimization problem in Eq. \ref{eq: chordal flag mean opt} can be phrased as the Stiefel manifold optimization problem: 
\begin{equation} \label{eq: stiefel opt}
    \argmin_{\Y \in St(d_k,d)} \sum_{j=1}^k m_j -  \tr \left( \I_j \Y^{\top} \mathbf{P}_j  \Y \right). 
\end{equation}
where the matrices $\I_j$ and $\mathbf{P_j}$ are given below
\begin{equation}\label{eq: Ij}
    (\I_j)_{i,l} = 
    \begin{cases}
        1, & i = l \in \{ d_{j-1} + 1, 
 d_{j-1} + 2, \dots, d_j\} \\
        0, &\text{ otherwise}\nonumber \\
    \end{cases},
\end{equation}
\begin{equation}\label{eq: Pj}
    \mathbf{P}_j =  \sum_{i=1}^p \alpha_j \X_j^{(i)} {\X_j^{(i)}}^{\top}.
\end{equation}
\end{prop}
\begin{proof}[Proof sketch]
We use truncated orthogonal representations for points on the Stiefel and flag manifolds.
By the equivalence of minimization problems we write Eq.~\ref{eq: chordal flag mean opt} as
\begin{equation*}
 \argmin_{\Y \in St(d_k,d)}\sum_{j=1}^k  m_j -  \sum_{j=1}^k \sum_{i=1}^p \alpha_i \tr\left( \Y_j^{\top}\X_j^{(i)} {\X_j^{(i)}}^{\top}\Y_j \right).
\end{equation*}
$\I_j$ allows us to write $\Y_j \Y_j^{\top} = \Y \I_j \Y^{\top}$. Using this, properties of trace, and our definition of $\mathbf{P}_j$ we write Eq.~\ref{eq: chordal flag mean opt} as Eq.~\ref{eq: stiefel opt}.

\end{proof}
We provide the full proof in 
the appendix. We now extend the chordal mean to the case of a certain family of \emph{complete} and \emph{oriented} \emph{flags}.




\begin{prop}\label{prop:eucmean}
Let $ \{ \x^{(i)} \}_{i=1}^p \subset \R^d$. Then suppose ${\x^{(i)}}^{\top} \x^{(j)} > 0$ for all $i,j$. Then the naive Euclidean mean $\bm{z} = \frac{1}{n}\sum_{i=1}^p\x^{(i)}$ has the same orientation as each $\x^{(i)}$. 
\end{prop}
\begin{proof}
    The proof follows from the simple derivation:
\begin{align*}
{\x^{(j)}}^{\top} \bm{z} = {\x^{(j)}}^{\top} \frac{1}{n}\sum_{i=1}^p\x^{(i)}
= \frac{1}{n}\sum_{i=1}^p {\x^{(j)}}^{\top} \x^{(i)} > 0.
\end{align*}
\end{proof}
\begin{dfn}[$\flag^+(1,\dots,d-1;d)$ chordal flag-mean]\label{def:reorient}
Let $\{[\![\X^{(i)}]\!]\}_{i=1}^p \subset \flag(1,2,\dots,d-1;d)$  where for each $j$ and any $i$ and $k$, ${\X_j^{(i)}}^{\top}{\X_j^{(k)}} > 0$. Let $[\![\bm{\mu}]\!]$ be the chordal flag-mean (e.g., Eq.~\ref{eq: chordal flag mean opt}) and $\bm{z}_j$ be the Euclidean mean of $\{ \X_j^{(i)}\}_{i=1}^p \in \R^d$. Then the oriented chordal flag-mean is defined as $[\![\bm{\mu}^+]\!] \in \flag^+(1,\dots,d-1;d)^+$:
\begin{equation}\label{eq: orientation ex}
    \bm{\mu}^+_j =
    \begin{cases}
        \Y_j, & \bm{z}_j^{\top} \Y_j \geq 0\\
        -\Y_j, & \text{otherwise}.
    \end{cases}
\end{equation}
\end{dfn}
\begin{remark}
The ordering of the columns of $\bm{\mu}$ is the same as that of each $\X^{(i)}$ because the chordal distance on the flag manifold respects the ordering of the vectors in the flag representation by only comparing $\bm{\mu}_j$ to $\X_j^{(i)}$. So, we only need to correct for the sign of the columns of $\bm{\mu}$.
By Prop.~\ref{prop:eucmean}, we know that the Euclidean mean, $\bm{z}$, 
has the same orientation as each of $\X_j^{(i)}$. We use Eq.~\ref{eq: orientation ex} to force $\bm{\bm{z}}_j^{\top}\bm{\mu}_j^* \geq 0$.
Dfn. \ref{def:reorient} gives us a way to choose which chordal flag-mean representatives are best for averaging representations of motions in $\flag^+(1,2,3;4)$ in Sec.~\ref{sec:motavg}.
\end{remark}




\begin{algorithm}[t]
\setstretch{1.13}
\caption{Chordal flag-mean.}\label{alg:chordalavg}
 \textbf{Input}: {Set of points on a flag manifold $\{[\![\X^{(i)}]\!]\}_{i=1}^p$}\\
 \textbf{Output}: Chordal flag-mean {$[\![\bm{\mu}]\!]$} \\[0.25em]
 Initialize $[\![\bm{\mu}]\!]$\\
 Compute projections ${\{\mathbf{P}_i\}_{i=1}^k}$ as in Eq.~\ref{eq: Pj}\\
 Define ${\{\mathbf{I}_i\}_{i=1}^k}$ as in~\cref{eq: Ij}\\
 Optimize~\cref{eq: stiefel opt} using Stiefel RTR to find $[\![\bm{\mu}]\!]$
\end{algorithm}

To compute the proposed mean, we optimize Eq.~\ref{eq: stiefel opt} via RTR methods~\cite{absil2007trust,boumal2014manopt} and re-orient the mean using Dfn.~\ref{def:reorient}.
\begin{remark}
    The geodesic distance averages on the Grassmannian (e.g. $\ell_2$-median and Karcher mean) are known to be unique only for certain subsets of the Grassmannian \cite{afsari2011riemannian}. The proof of this revolves around finding the region of convexity of the geodesic distance function and its square.
    Uniqueness for Grassmannian chordal distance averages (\eg the GR-mean \cite{draper2014flag} and -median \cite{mankovich2022flag}) is largely unstudied. It is known that the chordal distance on the Grassmannian approximates the geodesic distance, but its region of convexity 
    is an open problem to the best of our knowledge.
    Determining the convexity of our chordal flag-mean and -median would boil down to finding the region of convexity of the chordal distance function and its square on the flag manifold. Additionally, one could generalize geodesic distance averages to the flag manifold using Riemannian operators on flags~\cite{ye2022optimization}, find an algorithm to compute them and their region of convexity. We leave these projects to future work.
\end{remark}

\subsection{The Chordal Flag-median}
We are now ready to provide our iterative algorithm for robust centroid estimation.
\begin{dfn}[Weighted chordal flag-median]
Let $\{ [\![\X^{(i)}]\!] \}_{i=1}^p \subseteq \flag(d+1)$ be a set of points on a flag manifold with weights $\{\alpha_i\}_{i=1}^p \subset \R$ where $\alpha_i \geq 0$. The chordal flag-median, $[\![\bm{\eta}]\!]$, of these points solves
\begin{equation} \label{eq: chordal flag median opt}
    \argmin_{[\![\Y]\!] \in \flag(d+1)} \sum_{i=1}^p \alpha_i d_c([\![\X^{(i)}]\!], [\![\Y]\!]).
\end{equation}
\end{dfn}
\textit{Note: for $\flag(k;n)$, this amounts to the Grassmannian-median by Mankovich~\etal~\cite{mankovich2022flag}.}


\begin{prop}\label{prop: flag median}
The flag-median optimization problem in Eq.~\ref{eq: chordal flag median opt} can be phrased with weights $w_i([\![\Y]\!])$ in: 
\begin{equation}\label{eq: irls weights}
    w_i([\![\Y]\!]) =  \frac{\alpha_i}{\max\{d_c([\![\X^{(i)}]\!], [\![\Y]\!]), \epsilon\}},
\end{equation}
\begin{equation}\label{eq: median equiv}
 \argmin_{[\![\Y]\!] \in \flag(d+1)}\sum_{i=1}^p \sum_{j=1}^k m_j - w_i([\![\Y]\!]) \tr\left( \Y_j^{\top} \X_j^{(i)}{\X_j^{(i)}}^{\top} \Y_j \right).
\end{equation}
where $\epsilon = 0$ as long as $d_c([\![\X^{(i)}]\!], [\![\Y]\!]) \neq 0$ for all $i$.
\end{prop}

\begin{proof}[Proof sketch] 
We can encode the constraints and our optimization problem into the Lagrangian:
\begin{align}
\begin{aligned}
    \nabla_{\Y_j} \mathcal{L}
    &=  -2 \sum_{i=1}^p \frac{\alpha_i \X_j^{(i)}{\X_j^{(i)}}^{\top} \Y_j}{\sqrt{\sum_{j=1}^k m_j - \tr \left( {\X_j^{(i)}}^{\top} \Y_j \Y_j^{\top} \X_j^{(i)} \right) } } \nonumber\\
    &+2\sum_{j=1}^k \lambda_{i,j}\Y_i\Y_i^{\top} \Y_j,\\
     \nabla_{\lambda_{i,j}} \mathcal{L} &=  m_j \delta_{i,j} - \tr\left( \Y_i^{\top}\Y_j\Y_j^{\top} \Y_i \right).
\end{aligned}
\end{align}
Then we take the gradient of the Lagrangian with respect to $\Y_j$ and $\lambda_{i,j}$ and set it equal to zero. So, for each $j$, we have 
\begin{equation*}
4m_j \lambda_{j,j} =  \sum_{i=1}^p \frac{\alpha_i \tr \left( \Y_j^{\top} \X_j^{(i)}{\X_j^{(i)}}^{\top} \Y_j \right) }{d_c([\![\X^{(i)}]\!], [\![\Y]\!]) }.
\end{equation*}
Maximizing each $4m_j \lambda_{j,j}$ will minimize the objective function in Eq.~\ref{eq: chordal flag median opt}. We use equivalences of optimization problems to reformulate this maximization as Eq.~\ref{eq: median equiv}.
\end{proof}

\begin{prop}\label{prop: irls iteration}
 Fixing $[\![\Z]\!] \in \flag(d+1)$, Eq.~\ref{eq: median equiv}, with $w_i([\![\Z]\!])$, becomes
 \begin{equation*}
  \argmin_{[\![\Y]\!] \in \flag(d+1)}\sum_{i=1}^p \sum_{j=1}^k m_j - w_i([\![\Z]\!]) \tr\left( \Y_j^{\top} \X_j^{(i)}{\X_j^{(i)}}^{\top} \Y_j \right)
 \end{equation*}
 and is equivalent to a chordal flag-mean with weights $w_i([\![\Z]\!])$.
 Note: $\epsilon = 0$ as long as $d_c([\![\X^{(i)}]\!], [\![\Z]\!]) \neq 0$ for all $i$.
\end{prop}
\begin{proof}[Proof sketch] 
This follows from the proof of Prop.~\ref{prop: flag mean}.
\end{proof}

\begin{algorithm}[t]
\setstretch{1.13}
\caption{Chordal flag-median.}\label{alg:chordalmedian}
 \textbf{Input}: {Set of points on a flag manifold $\{[\![\X^{(i)}]\!]\}_{i=1}^p$}\\
 \textbf{Output}: Chordal flag-median {$[\![\bm{\eta}]\!]$} \\[0.25em]
 Initialize $[\![\bm{\eta}]\!]$\\
 \While{(not converged)}
 {
    Assign $w_i([\![\bm{\eta}]\!])$ using~\cref{eq: irls weights} (with $\epsilon>0$)\\
    $[\![\bm{\eta}]\!] \gets$ flag-mean($\{[\![\X^{(i)}]\!]\},\{w_i([\![\bm{\eta}]\!])\}$)
 }
\end{algorithm}
Prop.~\ref{prop: flag median} simplifies our optimization problem to Eq.~\ref{eq: median equiv}. Given an estimate for the chordal flag-median, $[\![\Z]\!]$, Prop.~\ref{prop: irls iteration} shows that solving a weighted chordal flag mean problem will approximate the solution to Eq.~\ref{eq: median equiv}. Using the propositions, we are now ready to present our iterative algorithm for flag-median estimation in Alg.~\ref{alg:chordalmedian}.

The convergence of Weiszfeld-type algorithms are well studied in the literature~\cite{aftab2015convergence,beck2015weiszfeld,zhao2020quaternion} and our IRLS algorithm for the chordal flag-median can be proven to  decrease its respective objective function value over iterations. This is what we establish next in Prop.~\ref{prop: flagirls decreasing}, inspired by the proof methods given in~\cite{beck2015weiszfeld}.

\begin{prop}\label{prop: flagirls decreasing}
Let $[\![\Y]\!] \in \flag(d+1)$. Suppose $d([\![\Y]\!],[\![\X^{(i)}]\!]) > \epsilon$ for $i = 1,2, \dots, p$. Also define the maps: $T:\flag(d+1) \rightarrow \flag(d+1)$ as an iteration of Alg.~\ref{alg:chordalmedian} and $f:\flag(d+1) \rightarrow \R$ as the chordal flag-median objective function value. Then
\begin{equation}
f(T([\![\Y]\!])) \leq f([\![\Y]\!]).
\end{equation}
\end{prop}
\begin{proof}[Proof sketch] 
We define the function 
\begin{equation}
h([\![\Z]\!],[\![\Y]\!]) = \sum_{i=1}^p w_i([\![\Z]\!]) d_c([\![\X^{(i)}]\!],[\![\Y]\!])^2.
\end{equation}
By definition of $h$, $T$, and $f$, we have
\begin{equation}
    h(T([\![\Y]\!]),  [\![\Y]\!]) \leq h([\![\Y]\!], [\![\Y]\!]) \leq f([\![\Y]\!]).\nonumber
\end{equation}
We use $h$ and $2a- b < \frac{a^2}{b}$ for $a,b \in \R$, $b > 0$ to find 
\begin{equation}
2f(T([\![\Y]\!]) - f([\![\Y]\!]) \leq h(T([\![\Y]\!]),  [\![\Y]\!]).\nonumber
\end{equation}
From our string of inequalities, we have the desired result. We leave the full proof to our supplementary material.
\end{proof}
\begin{remark}
The distance vanishes when $[\![\Y]\!] = [\![\X^{(i)}]\!]$ (e.g., $d_c([\![\Y]\!],[\![\X^{(i)}]\!]) =0$). In this case, Alg.~\ref{alg:chordalmedian} gets stuck at $[\![\X^{(i)}]\!]$ and the result in Prop.~\ref{prop: flagirls decreasing} becomes
\begin{equation}
f(T([\![\Y]\!])) \leq f([\![\Y]\!]) + {p \epsilon}/{2}.
\end{equation}
This singularity can be removed even for a general Weiszfeld iteration, simply by replacing the weights~\cite{aftab2014generalized}.
\end{remark}


\begin{prop}\label{cor: flagirls obj converges}
Let $[\![\Y_k]\!] \in \flag(d+1)$ be an iterate of Alg.~\ref{alg:chordalmedian} and $f:\flag(d+1) \rightarrow \R$ denote the chordal flag-median objective value. $f([\![\Y_k]\!])$ converges as $k \rightarrow \infty$ as long as $d_c([\![\Y]\!],[\![\X_i]\!]) >\epsilon$ for $i=1,2,\dots,p$ and each $k$.
\end{prop}

\begin{proof}
 Notice that the real sequence with terms $f([\![\Y_k]\!]) \in \R$ is bounded below by $0$ and is decreasing by Prop.~\ref{prop: flagirls decreasing}. So it converges as $k \rightarrow \infty$.
\end{proof}




\section{Motion Averaging}
\label{sec:motavg}
In this section, we propose a method for motion averaging by leveraging novel definitions of averages on the flag manifold. This will also act as a good example of how to use flag manifolds for performing computations on other groups. To this end, we now define the group of $3$D rotations and translations, $SE(3)$. Then we outline how to navigate between points on $SE(3)$ and points on a flag. Finally, we describe our \emph{motion averaging on flag manifolds}. 

\begin{dfn}[3D motion]
The configuration (position and orientation) of a rigid body moving in free space can be described by a homogeneous transformation matrix $\M$ corresponding to the displacement from any inertial reference frame to another. The set of all such rigid body transformations in three-dimensions form the $SE(3)$ group:
\begin{equation}
\label{eq:se3}
SE(3) = \left \{\bm{\gamma}:=
\begin{bmatrix} 
\Rot & \tb \\ 
\zero^\top & 1 
\end{bmatrix}\,:\, \Rot\in SO(3) \,\mathrm{\,and\,}\, \tb\in\R^3
\right\},\nonumber
\end{equation}
where $\tb$ denotes a \emph{translation} (positional displacement) and $\Rot$ captures the angular displacements as an element of the special orthogonal group $SO(3)$:
\begin{equation}
\label{eq:so3}
SO(3) = \left \{
\Rot \in \R^{3\times 3} \colon \Rot^\top\Rot = \Id \,\wedge\, \det \Rot = 1
\right\}.
\end{equation}
\end{dfn}

\begin{prop}[Motion contraction~\cite{ozyesil2018synchronization}]\label{prop:se3so4} We call \\ $\Phi_{\lambda}:SE(3)\to SO(4)$ a \emph{Saletan contraction}, \ie $\Phi_{\lambda}(\bm{\gamma}) = \U \V^T$ where the left ($\U$) \& right ($\V$) singular vectors are obtained via the singular value decomposition:
\begin{equation}
\U \bm{\Sigma} \V^T = \begin{bmatrix} \Rot & \tb/\lambda  \\ 
\zero^\top & 1 \end{bmatrix} \text{ \emph{for} } \bm{\gamma} \in SE(3).
\end{equation}
\end{prop}

\begin{prop}[Inverse motion contraction~\cite{ozyesil2018synchronization}]\label{prop:so4tose3}
We call the inverse contraction map $\Phi_{\lambda}^{-1}: SO(4)\to SE(3)$. Let $\mathbf{M} \in SO(4)$, then $\bm{\gamma} = \Phi_{\lambda}^{-1}(\mathbf{M})$ where
\begin{align}
    \tb &=  \frac{2\lambda}{\mathbf{M}_{4,4}}\mathbf{M}_{1:3,4},\\
    \Rot &= \begin{cases}
    \mathbf{M}_{1:k, 1:k}, & 
\| \tb \|_2 < \epsilon\\
    \left( \mathbf{M}_{4,4}\frac{\tb \tb^T}{\| \tb\|_2^2} +  \mathbf{P'} \right)^{-1} \mathbf{M}_{1:k,1:k}, & \mathrm{o.w.}
    \end{cases},
\end{align}
and $\U \bf{\Sigma} \V^T = \tb^T$ is the SVD and $\mathbf{P'} = \V_{:,2:4}\V_{:,2:4}^T$.
\end{prop}

\begin{algorithm}[t]
\setstretch{1.13}
\caption{Motion averaging on Flag manifolds.}\label{alg:motionavg}
 \textbf{Input}: {Motions $\{\bm{\gamma} \}_{i=1}^p \subset SE(3)$, scale $\lambda \in R$}\\
 \textbf{Output}: Average motion $\bm{\gamma^*} \in SE(3)$ \\[0.25em]
 Compute $\{\Phi_{\lambda}(\gamma_i)\}_{i=1}^p \subset SO(4)$ using~\cref{prop:se3so4}\\
 Compute $\left\{[\![\X^{(i)}]\!]\right\}_{i=1}^p \subset FL^+(1,2,3;4)\}$ from $\{\Phi(\gamma_i)\}_{i=1}^p$ using~\cref{prop:so4flag}\\
 \textbf{Mean}:$\quad [\![\Y^*]\!] \gets \text{flag-mean}\left( \left\{ [\![\X^{(i)}]\!] \right\}_{i=1}^p\right)$\\
 \textbf{Median}: $[\![\Y^*]\!] \gets \text{flag-median}\left(\left\{[\![\X^{(i)}]\!]\right\}_{i=1}^p\right)$\\
 Use~\cref{prop:FLtoSO4} to compute $\M^\star\in SO(4)$\\
 Use~\cref{prop:so4tose3} to compute $\Rot^\star\in SO(3)$ and $\tb^\star\in\R^3$
\end{algorithm}

\begin{prop}[Flag representation of motion~\cite{selig2005study}]\label{prop:so4flag}

Any contracted motion $\mathbf{M} \in SO(4)$ can be represented as a point on the flag, $[\![\X]\!] \in \flag^+(1,2,3;4)$ as the first $3$ columns of $\mathbf{M}$. Namely, $[\![\X]\!]$ is
\begin{equation}\label{eq: x flag def}
 \left[\mathbf{m_1}\right] \subset \left[\mathbf{m_1}, \mathbf{m_2}\right] \subset \left[\mathbf{m_1}, \mathbf{m_2}, \mathbf{m_3} \right] \subset \R^4.
\end{equation}
\end{prop}
\begin{remark}
Note that the elements of the group of rigid body motions, $SE(3)$, which we represent by points on $SO(4)$, can be imagined as the points of a six-dimensional quadric in seven-dimensional projective space, $\mathbb{P}^7$, called the \emph{Study quadric}~\cite{selig2005study}. The well known dual quaternions are the very coordinates of this space. Such a bijection between $\mathbb{P}^{7}$ and $SO(4)$~\cite{nawratil2016fundamentals} is the reason why our free parameter $\lambda$ resembles the \emph{dual unit} $\varepsilon$ in dual quaternions~\cite{selig2005study,ablamowicz2004lectures,busam2016_iccvw}. 
Moreover, our flag manifold, $\flag^+(1,2,3;4)$ is homeomorphic to $SO(4)$. We leave the investigation of these deeper connections to future work.
\end{remark}

\begin{prop}[Motion representation of a flag \cite{selig2005study}]\label{prop:FLtoSO4}
Given $[\![\X]\!] \in \flag^+(1,2,3;4)$ with the same basis vectors from Prop.~\ref{prop:so4flag}, the corresponding point on $SO(4)$ is
\begin{equation}
\left[ \m_1, \m_2, \m_3, \z \right] \in SO(4),
\end{equation}
where $\z$ is found by running the Gram-Schmidt process to find a $4^\text{th}$ unit vector orthogonal to $\text{span}\{ \m_1, \m_2, \m_3\} $.
\end{prop}
\begin{figure}[t]
    \includegraphics[width=\linewidth, clip, trim={0 .12in 0 .12in}]{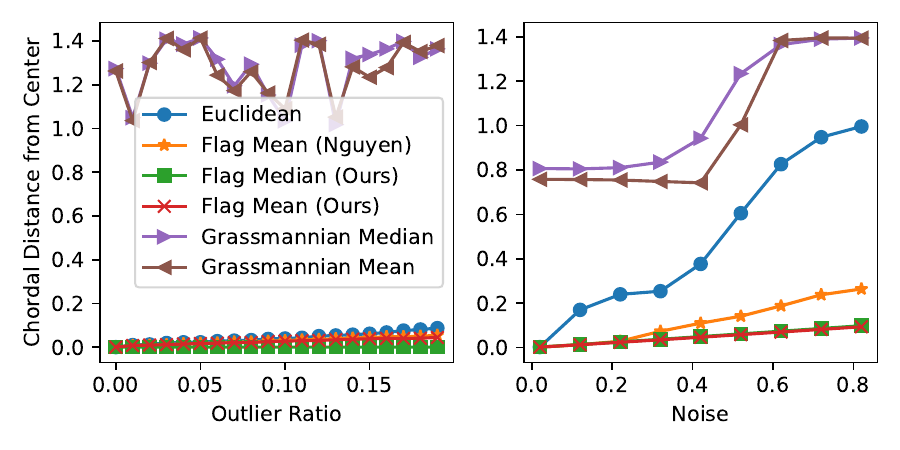}
    \caption{$100$ points from a synthetic data set on $\flag(1,3;10)$. The vertical axis is the chordal distance on $\flag(1,3;10)$ between the predicted averages and the ``center'' of the data set.\vspace{-4mm}}
    \label{fig:flagmeancompare}
\end{figure}

\subsection{Single Motion Averaging}
With these constructs, we are now ready to formally define the motion averaging problem for points on $SE(3)$.\vspace{0mm}
\begin{dfn}
Given a set of motions $\{\bm{\gamma}_i\in SE(3)\}_{i=1}^p$, the centroid is defined to be the solution of the following optimization procedure:
\begin{equation}
    \bm{\gamma}^* = \argmin_{\bm{\gamma}\in SE(3)} \sum_{i=1}^p \alpha_i \|\bm{\gamma}_i - \bm{\gamma}\|^q_\mathrm{F}
    \label{eq:motavg}
\end{equation}
where $q=2$ for mean estimation, $q=1$ for the median and $\alpha_i\in\mathbb{R}$ denote the individual weights.
\end{dfn}
To solve Eq.~\ref{eq:motavg}, we simply map each $\bm{\gamma}_i\in SE(3)$ to $\X^{(i)} \in FL(1,2,3;4)^+$. To this end, we first map each $\bm{\gamma}_i$ to $\phi_{\lambda}(\bm{\gamma}_i) = \M_i \in SO(4)$ via Prop.~\ref{prop:se3so4} and subsequently use Prop.~\ref{prop:so4flag} to represent $\M_i$ as 
$[\![\X^{(i)}]\!] \in FL(1,2,3;4)^+$. Then we use our flag-mean ($q=2$) or -median algorithm ($q=1$) to solve
\begin{equation}
    [\![\Y^*]\!] = \argmin_{[\![\Y]\!]\in FL(1,2,3;4)^+} \sum_{i=1}^p \alpha_i d_c([\![\X^{(i)}]\!], [\![\Y]\!])^q
    \label{eq:motavgflag}
\end{equation}
The desired solution $\bm{\gamma}^*\in SE(3)$ is then obtained by first mapping $[\![\Y^*]\!]$ back to $\M^*\in SO(4)$ via Prop.~\ref{prop:FLtoSO4} and subsequently using $\bm{\gamma}^* = \phi_{\lambda}^{-1}(\M^*)$ by Prop.~\ref{prop:so4tose3}. We present this chordal Flag motion averaging in Alg.~\ref{alg:motionavg}. 
\section{Results}\label{sec:results}
\subsection{Averaging on Flag Manifolds}
We first consider examples of data naturally existing as flags. 
We work with $5$ data sets: $2$ synthetic ones, MNIST digits \cite{deng2012mnist}, the Yale Face Database \cite{belhumeur1997eigenfaces}, and the Cats and Dogs dataset~\cite{yambor2000analysis}. We provide further evaluation of our flag averages that result in improved clustering on the UFC YouTube dataset~\cite{liu2009recognizing} in the supplementary material. In one synthetic experiment, we compare our Stiefel Riemannian Trust-Regions (RTR) method in Alg.~\ref{alg:chordalavg} for computing the flag-mean to the Flag RTR by Nguyen~\etal \cite{nguyen2022closed}. In the rest of the experiments, 
we compare our \textit{chordal} flag (FL)-mean \& -median to the Grassmannian (GR)-mean~\cite{draper2014flag} \& -median~\cite{mankovich2022flag}, as well as Euclidean averaging, where the matrices are simply averaged and projected onto the flag manifold via QR decomposition. 
GR-means and -medians, \cite{draper2014flag, mankovich2022flag} input data a points on Grassmannians by using the largest dimensional subspace in the flag ($[\X^{(i)}] \in \Gr(k,d)$) and output an average as a flag of type $(1,2,\dots,k,d)$. So all the methods considered in this section result in averages which live on a flag manifold. In this section we compare methods for data representation: the flag vs. Grassmannian vs. Euclidean space. 


\begin{table}[t]
\setlength{\tabcolsep}{3mm}
    \centering
    \begin{tabular}{l| c c}
         & Dist. to $\mathbf{C}$ & Obj. Fn. Value \\
         \toprule
        Ours & $(1.4 \pm 0.2) \times 10^{-4}$ & $(2.1 \pm 0.05) \times 10^{-4}$\\
         \cite{nguyen2022closed} & $(3.0 \pm 2.1) \times 10^{-3}$ & $(1.6 \pm 1.6) \times 10^{-3}$ \\
         \bottomrule
    \end{tabular}
    \caption{Robustness to initialization: Alg.~\ref{alg:chordalavg} versus Flag RTR from Nguyen~\etal \cite{nguyen2022closed}. Data: $100$ points on $\flag(1,2,3;10)$.\vspace{-4mm}}
    \label{tab:initrobust}
\end{table}

\paragraph{Synthetic data}
Both our synthetic experiments use the same methodology for generating data sets on the Grassmannian and flag. We begin by computing a ``center'' representative, $\mathbf{C} \in \R^{10 \times 3}$, as the first $3$ columns of the QR decomposition of a random matrix in $\R^{10 \times 3}$ with entries sampled from the uniform distribution over $[-.5,.5)$, $\mathcal{U}[-.5,.5)$. The representative for the $i^{\text{th}}$ data point, $\X_i$, is computed by sampling $\mathbf{Z}_i \in \R^{10 \times 3}$ with entries from $\mathcal{U}[-.5,.5)$ and defined as the first $3$ columns of the QR decomposition of $\mathbf{C} + \delta \mathbf{Z}_i$ for a noise parameter $\delta \geq 0$. 

\paragraph{Averaging synthetic flag data}
We use synthetic data sets with $100$ points, on $\Gr(3;10)$ and $\flag(1,3;10)$. For the left plot in Fig.~\ref{fig:flagmeancompare} we vary $\delta$ to compute our data sets. For the right plot we have $m$ outliers computed with $\delta = 1$ and the rest of the data are computed with $\delta = 0.001$. We compute the error as the chordal distance on $\flag(1,3;10)$ between the predicted average and $[\![\mathbf{C}]\!]$. In addition to comparing our averages to Grassmannian (GR) averages, we compare Alg.~\ref{alg:chordalavg} to Nguyen \etal~\cite{nguyen2022closed} for computing the flag-mean. Our results indicate that our algorithm improves both upon GR, Euclidean, and Nguyen~\etal~\cite{nguyen2022closed} averages in the sense that flag averages are closer to $[\![\mathbf{C}]\!]$. Specifically, our flag-median is more robust to outliers than our flag-mean. Note: Euclidean out preforms GR averaging because Euclidean averaging respects column order (e.g., the flag structure) for matrix representatives of the data, whereas GR averaging does not.  

\paragraph{Comparisons to Riemannian flag optimization}
In a second experiment, we compare the convergence of Alg.~\ref{alg:chordalavg} to that of Flag RTR~\cite{nguyen2022closed}. To this end, we generate $100$ points on $\flag(1,2,3;10)$ using $\delta = 0.001$ and run $50$ random trials with different initializations and compute $3$ items (i) the number of iterations to convergence, (ii) the chordal distance on $\flag(1,2,3;10)$ between the flag-mean and $[\![\mathbf{C}]\!]$, (iii) the cost function values from Eq.~\ref{eq: chordal flag mean opt}. We find that in every experiment Alg. \ref{alg:chordalavg} converges in \textit{2 iterations} and Flag RTR converges, on average, in $9.74 \pm 2.76$ iterations. In Tab.~\ref{tab:initrobust} we see that our method is one order of magnitude closer to the ground truth centroid $[\![\mathbf{C}]\!]$ and produces a one order of magnitude smaller objective function value.


\paragraph{Averaging under varying illumination}
To further demonstrate the efficacy of our averages over the standard Grassmanians, we leverage face images from Yale Face Database~\cite{belhumeur1997eigenfaces} with central ($c$), left ($l$), and right ($r$) illuminations, respectively. Let $\mathbf{A_c}, \mathbf{A_l}, \mathbf{A_r} \in \R^{243 \times 320}$ be these three images of a person.
We represent a face as a point $[\![\X]\!] \in \flag(1,3;d)$ as $[\![\X]\!] = [\X_1] \subset [\X] \subset \R^{d}$ and as $[\X] \in \Gr(3,d)$ using the following three steps: (i) Set $\mathbf{v_i} = \text{vec} \left(\mathbf{A_i} \right)$ for $i=c,l,r$; (ii) take $\X = \mathbf{Q}_{:,1:3}$ where $\mathbf{Q}$ is from the QR decomposition of $[\mathbf{v_c}, \mathbf{v_l}, \mathbf{v_r}]$.
Repeating this process for three faces gives us three points: $[\X_1],[\X_2],[\X_3] \in \Gr(3,d)$ and $[\![\X_1]\!],[\![\X_2]\!],[\![\X_3]\!] \in \flag(1,3;d)$.
Then we calculate the Grassmannian-mean of the points in $\Gr(3,d)$ which is the flag: $[\![\bm{\nu}]\!] = [\bm{\nu}_1] \subset [\bm{\nu}_1, \bm{\nu}_2] \subset [\bm{\nu}_1,\bm{\nu}_2,\bm{\nu}_3]$
and the flag-mean (ours) of the points in $\flag(1,3;d)$: $[\![\bm{\mu}]\!]  = [\bm{\mu}_1] \subset [\bm{\mu}_1,\bm{\mu}_2,\bm{\mu}_3] $.
A plot of reshaped $\bm{\mu}_1$ and $\bm{\nu}_1$ for a set of three faces in Fig.~\ref{fig:face_averages1}. We would expect the first dimension of both means to look like a face with center illumination. However, only the flag-mean appears to be center-illuminated. 

\begin{figure}[t]
        \includegraphics[width=\columnwidth]{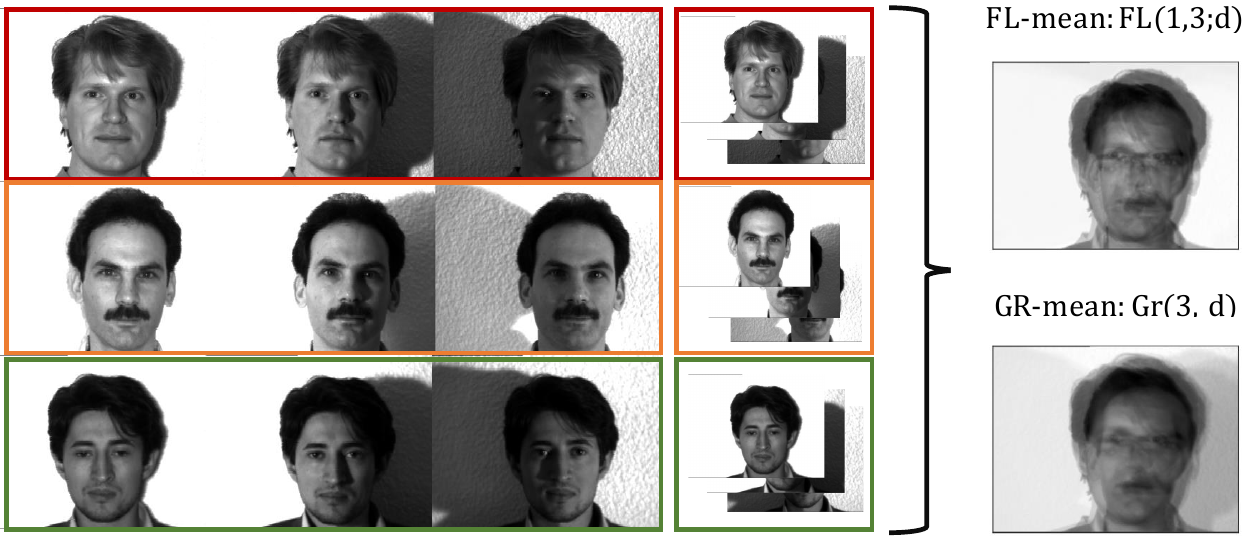}
	\caption{Averaging a collection of faces belonging to three different people, captured under varying illumination: center, left and right. Notice that the first dimension of the flag representations is center illuminated, better representing the mean compared to Grassmannian.\vspace{-4.5mm}}
	\label{fig:face_averages1}
\end{figure}

\paragraph{MNIST representation}
We run two experiments similar to what was done in~\cite{mankovich2022flag} with MNIST digits. However, our representations differ since we represent a digit as $[\X_j] \in \Gr(2, 784)$ and $[\![\X_j]\!] \in \flag(1,2;784)$. We generate $p$ representations of a digit, $\{\X_j\}_{j=1}^p$, by sampling a set of $p$ images without replacement from the test partition. Then we vectorize each image into $\mathbf{v}_j \in \R^{784}$ and run $k$ nearest neighbors on $\{\mathbf{v_j}\}_{i=1}^p$ with $k=2$ using the cosine distance. 
Say $\mathbf{v}_j$ and $\mathbf{v}_k$ are the $2$ nearest neighbors of $\mathbf{v}_j$, then the representation for sample $j$ is $\X_j = \mathbf{Q}_{:,:2}$ from the QR decomposition of $[\mathbf{v}_j, \mathbf{v}_k]$.

\paragraph{Robustness to Neural Network (NN) predictions} For the first MNIST experiment, we use the method above to create $20$ data sets on $\Gr(2, 784)$ and $\flag(1,2; 784)$ 
 corresponding to $i=0,1,2,\dots, 19$. The $i$th data set contains $20$ representations of the digit $1$ and $i$ representations for the digit $9$. We calculate a GR-mean and -median of each of the $i$ data sets on $\Gr(2, 784)$ and our flag-mean and -median for the data sets on $\flag(1, 2, 784)$. Note: all of these averages live on $\flag(1, 2, 784)$.
We then use a NN (trained on the original training data and producing a $97\%$ test accuracy on the original test data) to predict the label of the first dimension of each average for $i=0,1,2,\dots, 19$. As plotted in Fig.~\ref{fig: MNIST}, the NN incorrectly predicts the class of the GR-mean and -median for each data set. In contrast, the flag-mean and -median are all predicted as $1$s with data sets with fewer than $11$ representations of the $9$s digits. The flag-mean is the first flag average to be incorrectly predicted, since it is not as robust to outliers as the flag-median.

\begin{figure}[t]
    \includegraphics[width=\linewidth]{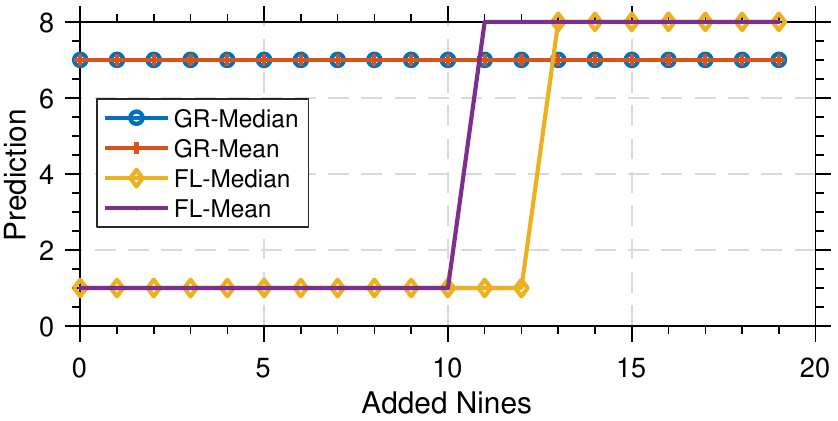}
    \caption{Neural network predictions for the first dimension of different averages $i=0,1,2,\dots,19$ MNIST data sets. The $i$th data set has $i$ representations of the $9$s digit and $20$ representations of the $1$s digit.\vspace{-5mm}}
    \label{fig: MNIST}
\end{figure}

\paragraph{Visualizing robustness}
Our second MNIST experiment is with $20$ representations of $6$s and with $i$ outlier representations of $7$s for $i=0,4,8,12$. We use the workflow from Fig.~\ref{fig: MNIST} to represent the MNIST digits on $\Gr(2,748)$ and $\flag(1,2;748)$. For each $i$, we compute averages of a data set with $i$ representations of $7$s. A chordal distance matrix on $\flag(1,2;798)$ between all the averages and data is used to preform Multidimensional Scaling (MDS)~\cite{kruskal1978multidimensional} for visualization in Fig.~\ref{fig:mnist_outliers}. 
The best averages should barely move (right to left) as we add outlier representations of $7$s. Our flag-mean and -median are moved the least with the addition of representations of $7$s with the median moving less than the mean. In contrast, the Grassmannian-mean and -median \cite{mankovich2022flag} move more than the compared baselines as we add $7$s.

\begin{figure}[t]
    \includegraphics[width=\linewidth]{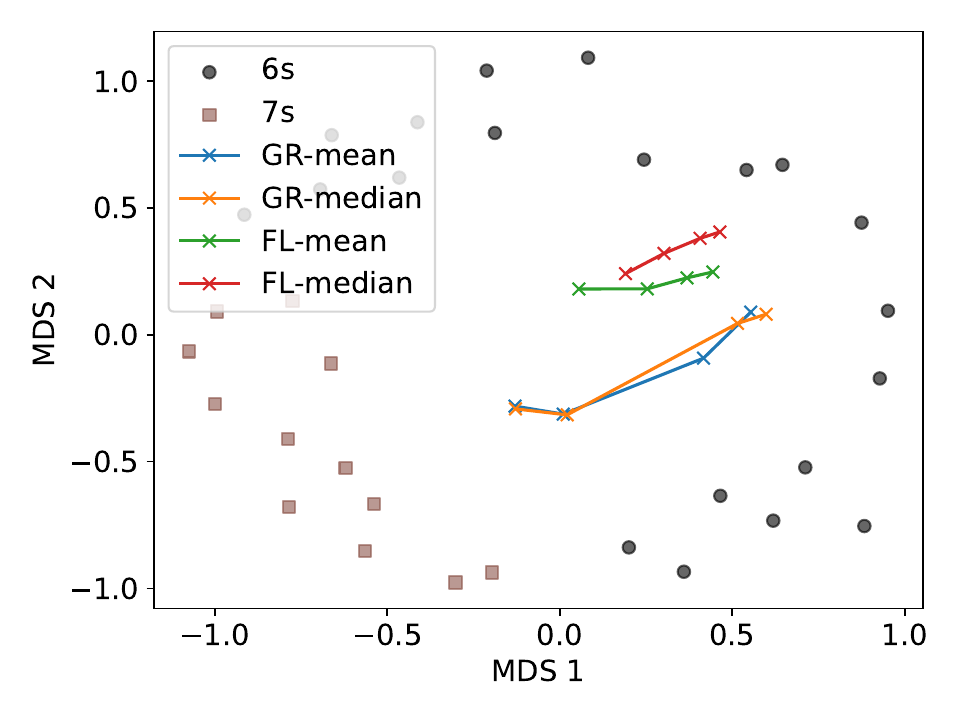}\vspace{-3mm}
    \caption{MDS embedding of MNIST digits and Grassmannian and flag averages. Each ``x'' is an average of $20$ representations of $6$s as we gradually add $i$ outlier representations of $7$s for $i=0,4,8,12$ data sets. The averages move from right to left as we add more $7$s.\vspace{-5mm}}
    \label{fig:mnist_outliers}
\end{figure}

\paragraph{PCA by flag statistics}
We use the Cats and Dogs dataset~\cite{yambor2000analysis} to compute $3$-dimensional PCA~\cite{hotelling1933analysis} weights, $\mathbf{W}^* \in \R^{4096 \times 3}$, of the data matrix, $\mathbf{X} \in \R^{198 \times 4096}$. Then we randomly split the $m$ subjects into $p$ evenly sized groups to generate $p$ data matrices each of size $p_i$: $\{\mathbf{X_i}\}_{i=1}^p \subset \R^{p_i \times 4096}$. PCA weights of each $\mathbf{X_i}$ are computed as $\mathbf{W_i}\in \R^{4096 \times 3}$. $\mathbf{W}^*$ is predicted by averaging $\{\mathbf{W_i}\}_{i=1}^p$ as points on $\flag(1,2,3; 4096)$ and $\Gr(3; 4096)$. Specifically, we compute the flag-mean (ours), Grassmannian-mean,  Euclidean-mean, and a random point. Then we record the chordal distance on $\flag(1,2,3; 4096)$ (reconstruction error) between the average and $[\![\mathbf{W}^*]\!] \in \flag(1,2,3; 4096)$. Our flag-mean is closer to $[\![\mathbf{W}^*]\!]$ for $p=1,2,\dots,6$.


\subsection{Averaging Rigid Motions}
We now evaluate our algorithm in robust averaging of a set of points represented on the $SE(3)$-manifold. To this end, we synthesize a dataset of 400 rigid motions (rotations and translations) around multiple randomly drawn central points in $SE(3)$. 
These points are generated with increasing noise levels. Particularly, for rotations we perturb the rotation axis using variances of $[0, 5, 10, 15, 20, 25]$ degrees,
    while the translations are perturbed in the levels of $[0, 0.02, 0.05, 0.1, 0.2, 0.3]$. For each noise level, we run 50 experiments and use $\lambda=1$ to ensure that translations and rotations are well balanced.
We then run our algorithms for the flag-mean and -median. These algorithms are compared to standard Govindu~\cite{govindu2004lie}, and baseline (QT) where translations and quaternions are averaged independently using Markley's method~\cite{markley2007averaging}. We also ran dual quaternion averaging of Torsello~\etal~\cite{torsello2011multiview} and found it produced identical results to Govindu.
Our results in Fig.~\ref{fig:se3avg} show that both of our algorithms surpass classical motion averages with our flag-median producing more robust estimates. 
\begin{figure}[t]
    \includegraphics[width=\linewidth]{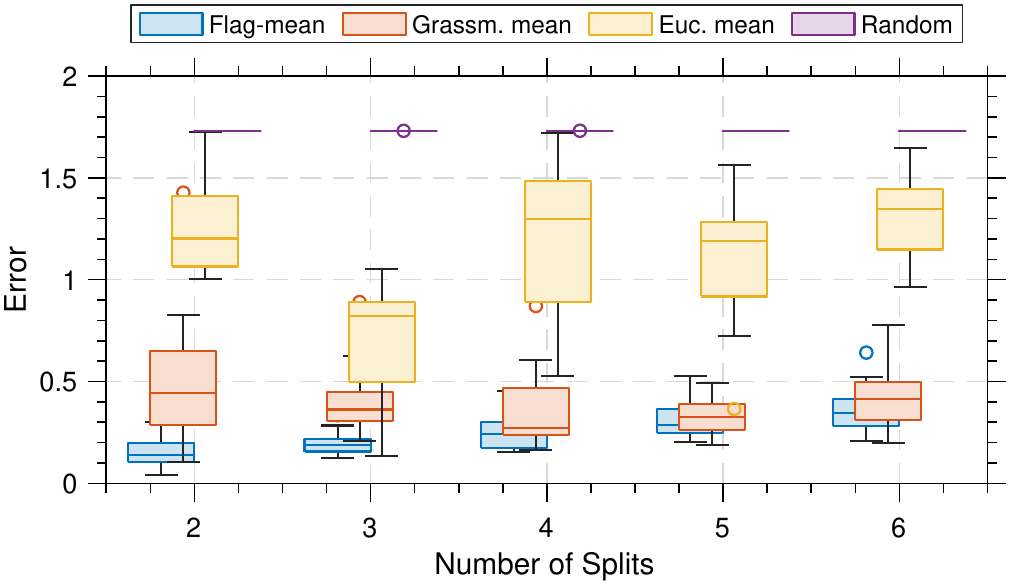}
    \vspace{-4mm}\caption{Reconstruction error for PCA weights as a function of Number of Splits, $p$.\vspace{-2mm}}
    \label{fig:pca_exp}
\end{figure}
\begin{figure}[t]
        \includegraphics[width=\columnwidth]{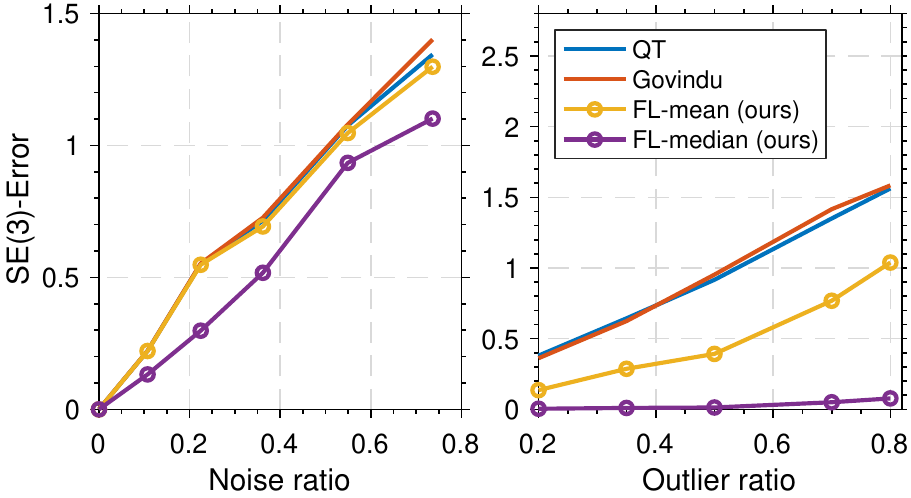}
	\vspace{-4mm}\caption{Single motion averaging experiments for increasing levels of $SE(3)$-noise and outlier ratios.\vspace{-4mm}}
	\label{fig:se3avg}
\end{figure}
\vspace{-2mm}\section{Conclusion}\vspace{-1mm}
We have provided two algorithms, the \emph{flag-mean} \& \emph{flag-median}, that estimate flag-prototypes of points defined on flag manifolds using chordal distance. We have established the convergence of our IRLS algorithm yielding the flag-median. Our methodologies deviate from the existing literature~\cite{draper2014flag, mankovich2022flag} which average Grassmannians into flags, and are found to be useful when either inherent outlier-robustness is necessary or when the subspaces possess a natural order, (e.g., hierarchical data). Since flag manifolds generalize Grassmannians, our methods can average on a broader class of manifolds. Consequently, we have applied our averages to rigid motions via group contraction.

\paragraph{Limitations \& future work} Our method can become computationally expensive when applied to high-dimensional problems. Moreover, our convergence results are weaker than desired as we have not provided a convergence rate. Besides addressing these, our future work involves clustering and inference on data with hierarchical structures.

\paragraph{Acknowledgements}
Benjamin Busam introduced Nathan and Tolga during CVPR 2022 in New Orleans. This gracious act is the catalyst in the realization of this work.

{\small

}

\vspace{5mm}
\appendix
\section*{Appendices}

\section{Flag Representations}
A flag is a nested collection of subspaces of increasing dimension. An illustration of a $\flag(1,2;3)$ is in Fig.~\ref{fig:flag}).

Flags are a natural representation for time series data as nested ``time subspaces.'' Suppose we data at three times: $\x_{t=1},\x_{t=2},\x_{t=3} \in \R^d$. We can group these data based on their ``effect over time'' in the sense that time $t=1$ stands alone, $t=1$ affects $t=2$, and $t=1$ and $t=2$ affect $t=3$. This grouping gives us the flag of type $\flag(1,2,3:d)$:\vspace{-2.5mm}
\begin{equation}
    \text{span}\{\x_1\}  \subset \text{span}\{\x_1,  \x_2 \} \subset \{ \x_1, \x_2, \x_3\} \subset \R^d.\vspace{-2.5mm}
\end{equation}
Flags can also model some hierarchical data using hierarchically nested subspaces. For a nice list of flags in mathematics, see~\cite{ye2022optimization}.

Recall $\flag(d+1) = \flag(d_1,d_2,\dots, d_k;d_{k+1}=d)$, we take $m_1=1$ and $m_j = d_j - d_{j-1}$. There are number of representations for flag manifolds involving quotients ~\cite{ye2022optimization}. We mention the most popular representation in the manuscript. A number works~\cite{ma2021flag, ma2022self, ye2022optimization} use
\begin{equation}\label{eq:sorep}
    \frac{SO(d)}{S(O(m_1) \times O(m_2) \times \cdots \times O(m_{k+1}))}
\end{equation}
where $S(O(m_1)\times \dots \times O(m_k))$ is\vspace{-2.15mm}
\begin{equation}
    \{ (\mathbf{M}_1, \dots, \mathbf{M}_k) \: : \: \prod_{i=1}^k \det(\mathbf{M}_i) = 1\}.\vspace{-1.8mm}
\end{equation}
Other works~\cite{pitaval2013flag, nguyen2022closed} represent flag manifolds using the quotient
\begin{equation}\label{eq:strep}
    \frac{St(d_k, d)}{O(m_1) \times O(m_2) \times \cdots \times O(m_{k})}.
\end{equation}
In this representation, $\X \in St(d_k, d)$ is used to represent the equivalence class
\begin{equation*}
[\![\X]\!] = \left\{ \X \mathbf{O} \: : \: \mathbf{O}_i \in O(m_i) \right\} \in \flag(d+1)
\end{equation*}
where  $\mathbf{O} = \text{diag}(\mathbf{O}_{m_1}, \dots, \mathbf{O}_{m_{k}})$. We use this Stiefel quotient representation in this manuscript.

\begin{figure}[t]
        \includegraphics[width=\columnwidth]{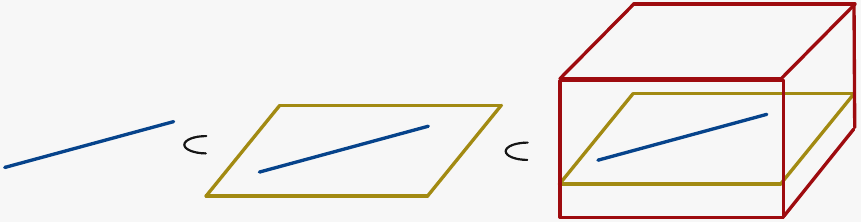}
	\caption{Illustration of a nested sequence of subspaces corresponding to
a point on the flag manifold. \vspace{-2.5mm}}
	\label{fig:flag}
\end{figure}
Ye~\etal prove that $\flag(d+1)$ is diffeomorphic to Eqs.~\ref{eq:sorep} and \ref{eq:strep} (see Prop. 4 and 12 in~\cite{ye2022optimization}). Additionally, Ye~\etal prove that flags are a closed submanifold of
\begin{equation}
    \Gr(m_1, d) \times \Gr(m_2, d) \times \cdots \Gr(m_k, d).
\end{equation}
Our chordal distance on flag manifolds leverages this product-of-Grassmannians since it is the $2$-norm of the chordal distances between each $\Gr(m_1, d)$.

\section{Proof of Proposition I}
Before providing the full proof, let us recall Prop. I:
\begin{prop}\label{prop: flag mean app}
The chordal flag-mean of $\left \{ [\![\X_i ]\!] \right \}_{i=1}^p \subset \flag(d+1)$ is
\begin{equation}\label{eq:suppflagmeanopt} 
    \left[\![\bm{\mu}\right]\!] := \argmin_{[\![\Y]\!]\in \flag(d+1)} \sum_{i=1}^p \alpha_i d_c([\![\X^{(i)}]\!], [\![\Y]\!])^2
\end{equation}
and can be phrased into a Stiefel manifold optimization problem as
\begin{equation}\label{eq:suppstiefelopt}
    \left[\bm{\mu}\right] = \argmin_{\Y \in St(d_k,d)} \sum_{j=1}^k m_j -  \tr \left( \I_j \Y^{\top} \mathbf{P}_j  \Y \right)
\end{equation}
where the matrices $\I_j$ and $\mathbf{P_j}$ are given in Eq. \ref{eq: Ij2} and Eq. \ref{eq: Pj2} respectively.
\begin{equation}\label{eq: Ij2}
    (\I_j)_{i,l} = 
    \begin{cases}
        1, & i = l \in \{ d_{j-1} + 1, 
 d_{j-1} + 2, \dots, d_j\} \\
        0, &\text{ otherwise} \\
    \end{cases}
\end{equation}
and define:
\begin{equation}\label{eq: Pj2}
    \mathbf{P}_j =  \sum_{i=1}^p \alpha_j \X_j^{(i)} {\X_j^{(i)}}^{\top}
\end{equation}
\end{prop}
For example, if we are averaging on $\flag(1,3;4)$ we have
\begin{equation*}
\I_1 = 
\begin{bmatrix}
    1 & 0 & 0\\
    0 & 0 & 0\\
    0 & 0 & 0
\end{bmatrix} \mathrm{ \quad and \quad}
\I_2 = 
\begin{bmatrix}
    0 & 0 & 0\\
    0 & 1 & 0\\
    0 & 0 & 1
\end{bmatrix}.
\end{equation*}
\begin{proof}
We begin by realizing Eq. \ref{eq:suppflagmeanopt} as an optimization problem using the definition of chordal distance:
\begin{equation*}
    \argmin_{[\![\Y]\!]\in \flag(d+1)} \sum_{i=1}^p  \alpha_i \left( \sum_{j=1}^k m_j - \tr\left( {\X_j^{(i)}}^{\top} \Y_j \Y_j^{\top} \X_j^{(i)} \right) \right).
\end{equation*}
Then we move our summations around to simplify our objective function:
\begin{align*}
    & \sum_{i=1}^p \alpha_i \left( \sum_{j=1}^k m_j - \tr\left( \Y_j^{\top}\X_j^{(i)} {\X_j^{(i)}}^{\top}\Y_j \right) \right)\\
    &=  \sum_{j=1}^k  \left( \sum_{i=1}^p \alpha_i\right) m_j -  \sum_{i=1}^p \alpha_i \tr\left( \Y_j^{\top}\X_j^{(i)} {\X_j^{(i)}}^{\top}\Y_j \right), \\
    &=  \sum_{j=1}^k  \left( \sum_{i=1}^p \alpha_i\right) m_j -  \sum_{j=1}^k \sum_{i=1}^p \alpha_i \tr\left( \Y_j^{\top}\X_j^{(i)} {\X_j^{(i)}}^{\top}\Y_j \right).
\end{align*}

Since $\sum_{i=1}^p \alpha_i$ is constant with respect to $[\![\Y]\!]$, Eq.~\ref{eq:suppflagmeanopt} is equivalent to 
\begin{equation*}
 \argmin_{[\![\Y]\!]\in \flag(d+1)} \sum_{j=1}^k  m_j -  \sum_{j=1}^k \sum_{i=1}^p \alpha_i \tr\left( \Y_j^{\top}\X_j^{(i)} {\X_j^{(i)}}^{\top}\Y_j \right).
\end{equation*}
Using our definitions for $\I_j$ and $\mathbf{P}_j$, we can write the objective function in terms of $\Y$:
\begin{align*}
    &\sum_{j=1}^k  m_j - \tr \left( \Y_j^{\top}\left( \sum_{i=1}^p\alpha_i\X_j^{(i)} {\X_j^{(i)}}^{\top} \right) \Y_j \right), \\
    &= \sum_{j=1}^k m_j - \tr \left( \Y_j^{\top}\mathbf{P}_j \Y_j \right),\\
    &= \sum_{j=1}^k m_j - \tr \left( \Y_j \Y_j^{\top} \mathbf{P}_j  \right),\\
    &= \sum_{j=1}^k m_j - \tr \left( \Y \I_j \Y^{\top} \mathbf{P}_j  \right).
\end{align*}
The third equality is true because $\Y_j \Y_j^{\top} = \Y \I_j \Y^{\top}$.


There are two constraints for $[\![\Y]\!] \in \flag(d+1)$ according to our representation for points on the flag manifold. The first constraint is $\Y_j^{\top} \Y_j = \I$ for $j=1,2,\dots,p$. The second constraint is $[\Y_j] \cap [\Y_i] = \emptyset$ for all $i \neq j$. These constraints are satisfied when $\Y^{\top}\Y = \I$, e.g. $\Y \in St(d_k,d)$.

Using trace invariance to cyclic permutations, the chordal flag mean optimization problem Eq.~\ref{eq:suppflagmeanopt} is equivalent to the Stiefel optimization problem Eq.~\ref{eq:suppstiefelopt}.
\end{proof}


\section{Proof of Proposition III}
\begin{prop}
The chordal flag-median of $\left \{ [\![\X_i ]\!] \right \}_{i=1}^p \subset \flag(d+1)$,
\begin{equation} \label{eq:suppflagmedian}
    \left[\![\bm{\eta}]\right] = \argmin_{[\![\Y]\!] \in \flag(d+1)} \sum_{i=1}^p \alpha_i d_c([\![\X^{(i)}]\!], [\![\Y]\!]),
\end{equation}
can be phrased with weights 
\begin{equation*}
    w_i([\![\Y]\!]) = \sum_{j=1}^k \frac{\alpha_i}{\max\{d_c([\![\X^{(i)}]\!], [\![\Y]\!]), \epsilon\}}
\end{equation*}
as the optimization problem
\begin{equation*}
 \argmin_{[\![\Y]\!] \in \flag(d+1)}\sum_{i=1}^p \sum_{j=1}^k m_j - w_i([\![\Y]\!]) \tr\left( \Y_j^{\top} \X_j^{(i)}{\X_j^{(i)}}^{\top} \Y_j \right)
\end{equation*}
with $\epsilon = 0$ as long as $d_c([\![\X^{(i)}]\!], [\![\Y]\!]) \neq 0$ for all $i$.
\end{prop}

\begin{proof}
We can write Eq.~\ref{eq:suppflagmedian} using the definition of chordal distance as
\begin{equation*}
     \argmin_{[\![\Y]\!] \in \flag(d+1)} \sum_{i=1}^p \alpha_i \sqrt{\sum_{j=1}^k m_j - \tr\left( {\X_j^{(i)}}^{\top} \Y_j \Y_j^{\top} \X_j^{(i)} \right) }.
\end{equation*}

The orthogonality constraints for $\Y \in \R^{d \times d_k}$ to represent a point on $\flag(d+1)$ are: (i) $[\Y_j] \cap [\Y_i]  = \emptyset$ for all $i \neq j$ and $\Y_j^{\top} \Y_j = \I$ for all $j$. Let $\theta([\Y_i],[\Y_j])$ denote the vector of principal angles between $[\Y_i]$ and $[\Y_j]$. Using $\tr(\Y_i^{\top}\Y_j\Y_j^{\top} \Y_i) = \|\cos \theta([\Y_i],[\Y_j])\|_2^2$, we encode our orthogonality constraints as
\begin{equation*}
\tr(\Y_i^{\top}\Y_j\Y_j^{\top} \Y_i) = \begin{cases}
                               0 & i \neq j\\
                               d_j & i = j
                              \end{cases}.
\end{equation*}
We will now use 
\begin{equation*}
    \delta_{i,j} = 
    \begin{cases}
        1, & i = j\\
        0, & i \neq j.
    \end{cases}
\end{equation*}
to put these constraints into the Lagrangian. 

Let $\bm{\Lambda}$ be a symmetric matrix of Lagrange multipliers corresponding to the orthogonality constraints. Denote the entry in the $i$th row and $j$th column of $\bm{\Lambda}$ as $\lambda_{i,j}$. With the constraints added to the objective, we define the Lagrangian in Eq.~\ref{eq: lagrangian}.
\begin{align}
\label{eq: lagrangian}
    \mathcal{L}(\Y, \Lambda) &=  \sum_{i=1}^p \alpha_i \sqrt{\sum_{j=1}^k m_j - \tr\left( {\X_j^{(i)}}^{\top} \Y_j \Y_j^{\top} \X_j^{(i)} \right)}\nonumber \\
    &-\sum_{i=j}^k \sum_{j=1}^k \lambda_{i,j} (m_j \delta_{i,j} - \tr(\Y_i^{\top}\Y_j\Y_j^{\top} \Y_i)).
\end{align}

The gradient of Eq.~\ref{eq: lagrangian} w.r.t. ${\Y_j}$ and ${\lambda_{i,j}}$ is
\begin{align*}
\begin{aligned}\label{eq: lagrangian grad}
    \nabla_{\Y_j} \mathcal{L}
    &=  - \sum_{i=1}^p \frac{\alpha_i \X_j^{(i)}{\X_j^{(i)}}^{\top} \Y_j}{\sqrt{\sum_{j=1}^k m_j - \tr \left( {\X_j^{(i)}}^{\top} \Y_j \Y_j^{\top} \X_j^{(i)} \right) } } \\
    &+2\sum_{\substack{i=1 \\ i \neq j}}^k \lambda_{i,j}\Y_i\Y_i^{\top} \Y_j + 4\lambda_{j,j}\Y_j\Y_j^{\top} \Y_j,\\
     \nabla_{\lambda_{i,j}} \mathcal{L} &=  m_j \delta_{i,j} - \tr\left(\Y_i^{\top}\Y_j\Y_j^{\top} \Y_i\right).
\end{aligned}
\end{align*}
Notice we are not dividing by zero because $d_c([\![\X^{(i)}]\!], [\![\Y]\!]) \neq 0$ for all $i$.
Now we use $\nabla_{\Y_j} \mathcal{L} = \bm{0}$ and $\nabla_{\lambda_{i,j}} \mathcal{L} = 0$ to solve for $\lambda_{j,j}$. 

First we will work with $\nabla_{\Y_j} \mathcal{L} = \bm{0}$.
\begin{align*}
    \begin{aligned}
        \mathbf{0} &=  - \sum_{i=1}^p \frac{\alpha_i \X_j^{(i)}{\X_j^{(i)}}^{\top} \Y_j}{d_c([\![\X^{(i)}]\!], [\![\Y]\!])} \\
    &+2\sum_{\substack{i=1 \\ i \neq j}}^k \lambda_{i,j}\Y_i\Y_i^{\top} \Y_j + 4\lambda_{j,j}\Y_j\Y_j^{\top} \Y_j,\\
    &=  - \sum_{i=1}^p \frac{\alpha_i \Y_j^{\top} \X_j^{(i)}{\X_j^{(i)}}^{\top} \Y_j}{d_c([\![\X^{(i)}]\!], [\![\Y]\!])} \\
    &+2\sum_{\substack{i=1 \\ i \neq j}}^k \lambda_{i,j}\Y_j^{\top}\Y_i\Y_i^{\top} \Y_j + 4\lambda_{j,j}\Y_j^{\top}\Y_j\Y_j^{\top} \Y_j,\\
    0 &=  - \sum_{i=1}^p \frac{\alpha_i \tr \left( \Y_j^{\top} \X_j^{(i)}{\X_j^{(i)}}^{\top} \Y_j \right)}{d_c([\![\X^{(i)}]\!], [\![\Y]\!])} \\
    &+2\sum_{\substack{i=1 \\ i \neq j}}^k \lambda_{i,j}\tr \left( \Y_j^{\top}\Y_i\Y_i^{\top} \Y_j\right) + 4\lambda_{j,j}\tr \left(\Y_j^{\top}\Y_j\Y_j^{\top} \Y_j\right).\\
    \end{aligned}
\end{align*}
Using $\nabla_{\lambda_{i,j}} \mathcal{L} = 0$ simplifies our equation to
\begin{align*}
4\lambda_{j,j} \tr (\Y_j^{\top} \Y_j\Y_j^{\top} \Y_j) &=  \sum_{i=1}^p \frac{\alpha_i \tr \left( \Y_j^{\top} \X_j^{(i)}{\X_j^{(i)}}^{\top} \Y_j \right) }{d_c([\![\X^{(i)}]\!], [\![\Y]\!]) }, \\
4m_j \lambda_{j,j} &=  \sum_{i=1}^p \frac{\alpha_i \tr \left( \Y_j^{\top} \X_j^{(i)}{\X_j^{(i)}}^{\top} \Y_j \right) }{d_c([\![\X^{(i)}]\!], [\![\Y]\!]) }. \\
\end{align*}


For $[\![\Y]\!]$ to minimize Eq.~\ref{eq:suppflagmedian}, we would want to maximize $m_j \lambda_{j,j}$ for each $j$. That is to say, we wish to maximize $\sum_{j=1}^k m_j \lambda_{j,j}$:
\begin{equation}\label{eq:flagmedianmax}
\sum_{i=1}^p \sum_{j=1}^k \frac{\alpha_i}{d_c([\![\X^{(i)}]\!], [\![\Y]\!])} \tr\left( \Y_j^{\top} \X_j^{(i)}{\X_j^{(i)}}^{\top} \Y_j \right).
\end{equation}
Maximizing Eq.~\ref{eq:flagmedianmax} is the same as minimizing 
\begin{equation*}
\sum_{i=1}^p \sum_{j=1}^k m_j - \frac{\alpha_i}{ d_c([\![\X^{(i)}]\!], [\![\Y]\!])} \tr\left( \Y_j^{\top} \X_j^{(i)}{\X_j^{(i)}}^{\top} \Y_j \right).
\end{equation*}
Using the definition of $w_i([\![\Y]\!])$, this minimization is
\begin{equation*}
 \argmin_{[\![\Y]\!] \in \flag(d+1)}\sum_{i=1}^p \sum_{j=1}^k m_j - w_i([\![\Y]\!]) \tr\left( \Y_j^{\top} \X_j^{(i)}{\X_j^{(i)}}^{\top} \Y_j \right)
\end{equation*}

\end{proof}

\begin{prop}
 Fix $[\![\Z]\!] \in \flag(d+1)$. Then the minimizer of 
 \begin{equation}\label{eq: median equiv app}
  \sum_{i=1}^p \sum_{j=1}^k\left( m_j - w_i(\Z) \tr\left( \Y_j^{\top} \X_j^{(i)}{\X_j^{(i)}}^{\top} \Y_j \right)\right)
 \end{equation}
 over $[\![\Y]\!] \in \flag(d+1)$ is the weighted chordal flag mean of $\{ [\![ \X^{(i)}]\!]\}_{i=1}^p \in \flag(d+1)$ with weights $w_i(\Z)$.
 Note: $\epsilon = 0$ as long as $d_c([\![\X^{(i)}]\!], [\![\Z]\!]) \neq 0$ for all $i$.
\end{prop}

\begin{proof}
By re-arranging the summations in Eq.~\ref{eq: median equiv app}, we see its minimizer is also
 \begin{equation*}
 \argmin_{[\![\Y]\!] \in \flag(d+1)}\sum_{j=1}^k  m_j - \sum_{j=1}^k \sum_{i=1}^p w_i(\Z) \tr\left( \Y_j^{\top} \X_j^{(i)}{\X_j^{(i)}}^{\top} \Y_j \right).
\end{equation*}
We showed that this is the same as the chordal flag-mean optimization problem with weights $w_i(\Z)$ in the proof of Prop.~\ref{prop: flag mean app}.
\end{proof}

\section{Proof of Proposition VI}
\begin{prop}
Let $[\![\Y]\!] \in \flag(d+1)$ and $\epsilon > 0$. Assume that $d([\![\Y]\!],[\![\X^{(i)}]\!]) > \epsilon$ for $i = 1,2, \dots, p$. Denote the flag median objective function value as $f:\flag(d+1) \rightarrow \R$ and an iteration of our chordal flag-median IRLS algorithm as $T:\flag(d+1) \rightarrow \flag(d+1)$. Then
\[
f(T([\![\Y]\!])) \leq f([\![\Y]\!]).
\]
\end{prop}

\begin{proof}
Assuming that $d([\![\Y]\!],[\![\X^{(i)}]\!]) > \epsilon$ for $i = 1,2, \dots, p$, we define the function $h: \flag(d+1) \times \flag(d+1) \rightarrow \R$ as
\begin{align*}\label{eq: h flagirls def}
\begin{aligned}
h([\![\Z]\!], [\![\Y]\!]) &= \sum_{i=1}^p w_i([\![\Y]\!]) d_c([\![\Z]\!],[\![\X^{(i)}]\!])^2,\\  
w_i([\![\Y]\!]) &=  \frac{1}{\max \left\{ d_c([\![\Y]\!],[\![\X^{(i)}]\!]), \epsilon \right \} } \\
&= \frac{1}{ d_c([\![\Y]\!],[\![\X^{(i)}]\!]) }.
\end{aligned}
\end{align*}

Some algebra reduces $h([\![\Z]\!], [\![\Y]\!])$ to
\begin{align*}
    h([\![\Z]\!], [\![\Y]\!]) &= \sum_{i=1}^p w_i([\![\Y]\!]) d_c([\![\Z]\!],[\![\X^{(i)}]\!])^2,\\
    &= \sum_{i=1}^p \frac{d_c([\![\Z]\!],[\![\X^{(i)}]\!])^2}{ d_c([\![\Y]\!],[\![\X^{(i)}]\!]) }.\\
\end{align*}

$h([\![\Z]\!], [\![\Y]\!])$ is the weighted flag-mean objective function (of $\{[\![\X^{(i)}]\!]\}_i$) with weights $w_i([\![\Y]\!])$. So minimizing $h([\![\Z]\!], [\![\Y]\!])$ over $[\![\Z]\!]$ is an iteration of our IRLS algorithm to compute the flag-median. In other words,

\begin{equation}\label{eq: h flagirls min}
    T([\![\Y]\!]) = \argmin_{[\![\Z]\!] \in \flag(d+1)} h([\![\Z]\!], [\![\Y]\!]).
\end{equation}
Using Eq.~\ref{eq: h flagirls min}, we have
\begin{equation*}
    h(T([\![\Y]\!]), [\![\Y]\!]) \leq h([\![\Y]\!], [\![\Y]\!]).
\end{equation*}

By the definition of $h$
\begin{align*}
    h([\![\Y]\!], [\![\Y]\!]) &= \sum_{i=1}^p \frac{d_c([\![\Y]\!],[\![\X^{(i)}]\!])^2}{d_c([\![\Y]\!],[\![\X^{(i)}]\!])}, \\
    &= \sum_{i=1}^p d_c([\![\Y]\!],[\![\X^{(i)}]\!]),\\
    &= f([\![\Y]\!]).
\end{align*}
This means, we have
\begin{equation}\label{eq: h flagirls less}
    h(T([\![\Y]\!]),  [\![\Y]\!]) \leq f([\![\Y]\!]).
\end{equation}

Now we use the identity from algebra: $\frac{a^2}{b} \geq 2a-b$ for any $a,b \in \R$ and $b > 0$. Let 
\begin{equation*}
    a = d_c([\![\Z]\!],[\![\X^{(i)}]\!]) \text{ and } b =  d_c([\![\Y]\!],[\![\X^{(i)}]\!]) .
\end{equation*} 
Then
\begin{align*}
h([\![\Z]\!], [\![\Y]\!]) &\geq 2\sum_{i=1}^p d_c([\![\Z]\!],[\![\X^{(i)}]\!])\\
&- \sum_{i=1}^p  d_c([\![\Y]\!],[\![\X^{(i)}]\!]),  \\
&= 2f([\![\Z]\!]) - f([\![\Y]\!]).
\end{align*}
Now, take $[\![\Z]\!] = T([\![\Y]\!])$. This gives us
\begin{equation}\label{eg: h flagirls greater}
    h(T([\![\Y]\!]), [\![\Y]\!]) \geq 2f(T([\![\Y]\!]))-  f([\![\Y]\!]).
\end{equation}

Then, combining Eq.~\ref{eg: h flagirls greater} with Eq.~\ref{eq: h flagirls less}, we have

\begin{align*}
    2f(T([\![\Y]\!]))- f([\![\Y]\!]) &\leq f([\![\Y]\!]), \\
    f(T([\![\Y]\!])) &\leq f([\![\Y]\!]) . 
\end{align*}
\end{proof}

\section{Further Experimental Evaluation}
\subsection{Further Qualitative Results on Faces Dataset}
We now show in Fig.~\ref{fig:sra12} further visualizations of Flag and Grassmann averages of faces.  
\begin{figure}[t]
        \includegraphics[width=\columnwidth]{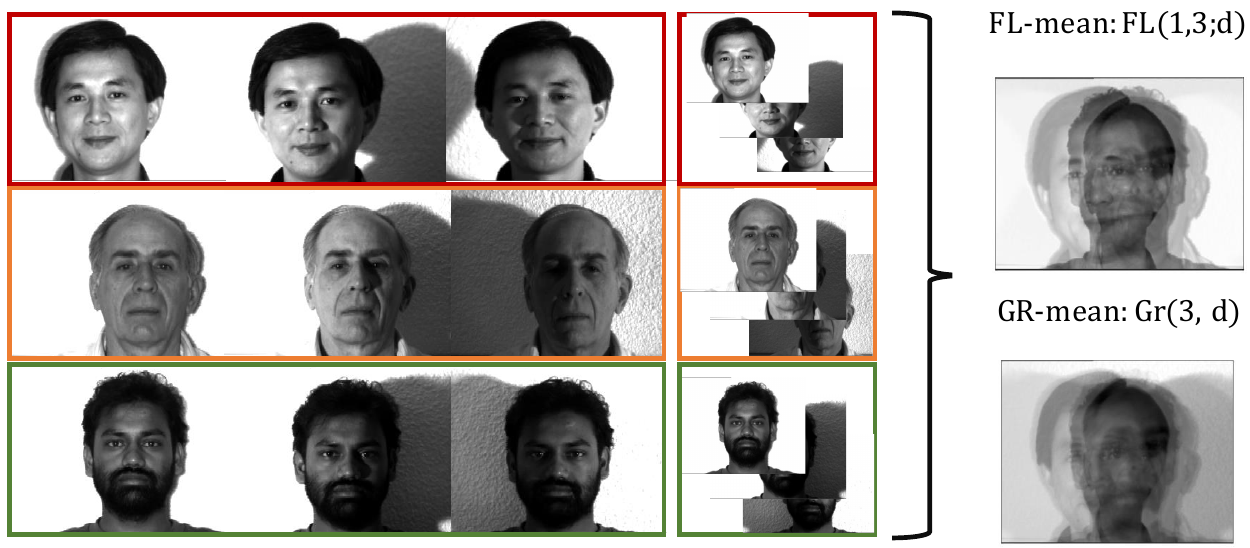}
	\caption{Averaging a collection of faces belonging to three different identities, captured under varying illumination: center, left and right. Notice that the first dimension of the flag representations is center illuminated, better representing the mean compared to Grassmannian.\vspace{-3mm}}
	\label{fig:sra12}
\end{figure}

\subsection{Further Qualitative Results on MNIST}
We use $20$ examples (e.g., points on $\flag(1,2;748)$) of $6$s and add $10$ examples of $7$s. We use the same workflow from the manuscript to represent the MNIST digits on $\Gr(2,748)$ and $\flag(1,2;748)$. We compute the averages on the Grassmannian~\cite{draper2014flag, mankovich2022flag} and flag (ours). The reshaped first dimension of each of these averages is in Fig.~\ref{fig:mnist_qual}. The brightness of the bottom left corner of each image is brighter the more present the $7$s digit (outlier class) is in the image. Notice the bottom left corner of each image, boxed in red, becomes darker as we move from left to right. So, our averaging on the flag is more robust to outliers than Grassmannian averaging. In fact, the bottom left corner of the flag-median is the darkest. Therefore, our flag-median is the least affected by the outlier examples of $7$s.
\begin{figure}[t]
    \includegraphics[width=\columnwidth]{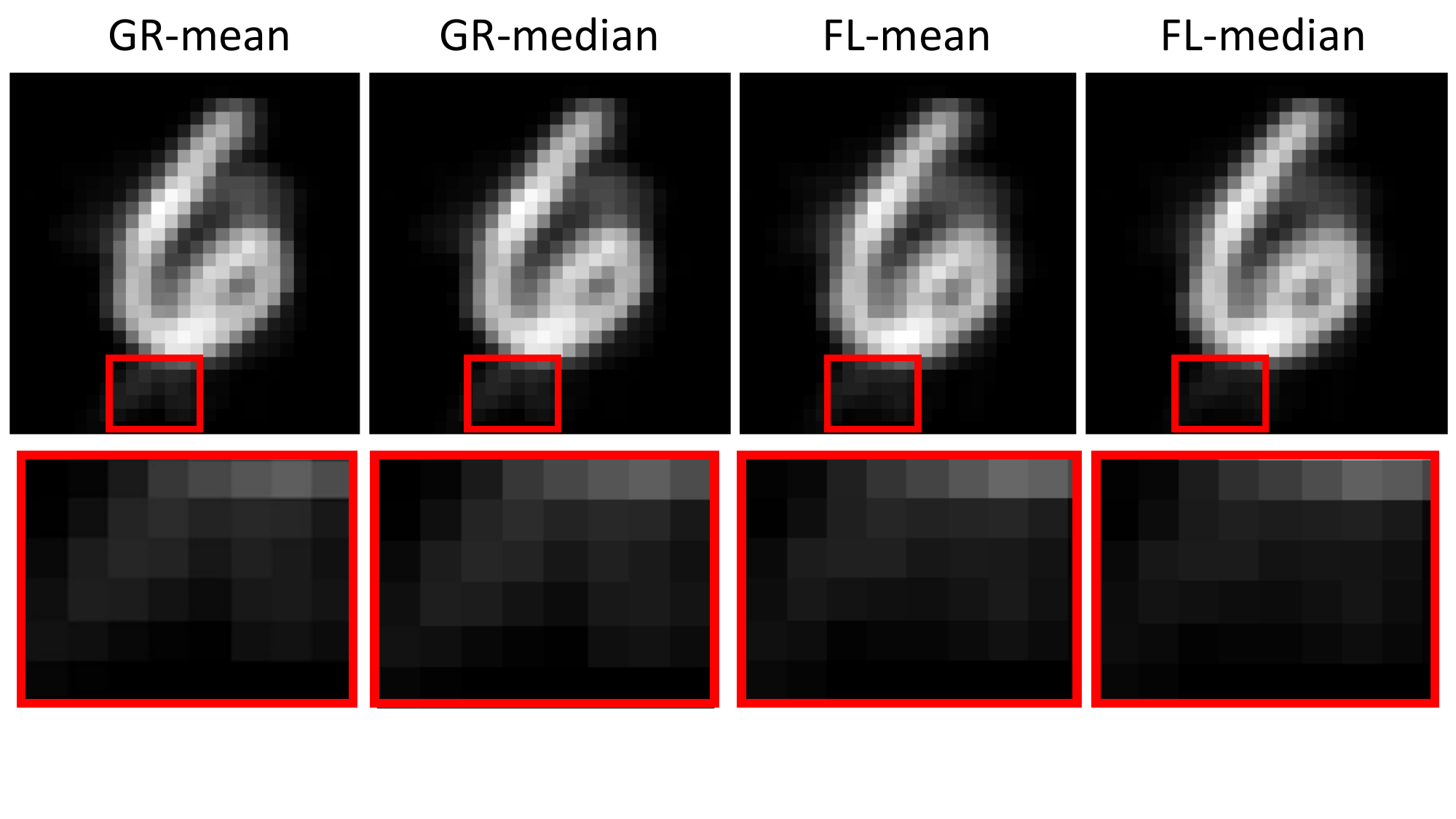}
	\caption{The first dimension of Grassmannian (``GR-'') and flag (``FL-'') averages of a data set with $20$ representations of $6$s and $10$ representations of $7$s. The bottom red boxes are the enlarged version of the upper image. Our flag-median is the least affected by the outlier examples of $7$s.}
	\label{fig:mnist_qual}
\end{figure}

\subsection{LBG Clustering on UFC YouTube}
 We use a subset of the UCF YouTube Action dataset~\cite{liu2009recognizing} to run a similar experiment to what was done by Mankovich~\etal \cite{mankovich2022flag}. The dataset contains labeled RGB video clips of people performing actions. Within each labeled action, the videos are grouped into subsets of clips with common features. We take approximately one example from each subset from each class. This results in $23$ examples of basketball shooting, $23$ of biking/cycling, $25$ of diving, $25$ of golf swinging, $24$ of horse back riding, $25$ of soccer juggling, $23$ of swinging, $24$ of tennis swinging, $24$ of trampoline jumping, 24 of volleyball spiking, and $24$ of walking with a dog. We convert these frames to gray scale, then we use INTER\_AREA interpolation from the OpenCV package~\cite{opencv_library} to resize the frames to have only $450$ pixels. This is, on average, only $1\%$ of the number of pixels in the original frame. We vectorize and horizontally stack each video, then use the first $10$ columns of $\mathbf{Q}$ from the QR decomposition to realize each video as a point on $\Gr(10,450)$ and $\flag(1,2,\dots,10;450)$.

 We run Linde-Buzo-Gray (LBG) clustering on these videos and report the resulting cluster purities in Fig.~\ref{fig:lbgres}. Clustering on the flag manifold with our flag averages (blue boxes) improves cluster purities over Grassmannian methods. We also see higher variance in cluster purities for flag methods. Even though we are only working with approximately $1\%$ of the total number of pixels in each frame, we are able to produce cluster purities that are competitive with those reported in~\cite{mankovich2022flag} using a similar set of videos. Specifically, our flag-LBG clustering is well within $0.1$ of the highest cluster purities reported in~\cite{mankovich2022flag}. Overall, our flag methods improve cluster purities in a head-to-head experiment while remaining competitive with Grassmannian LBG with only using approximately 1\% of pixels per frame.
\begin{figure}[t]
        \includegraphics[width=\columnwidth]{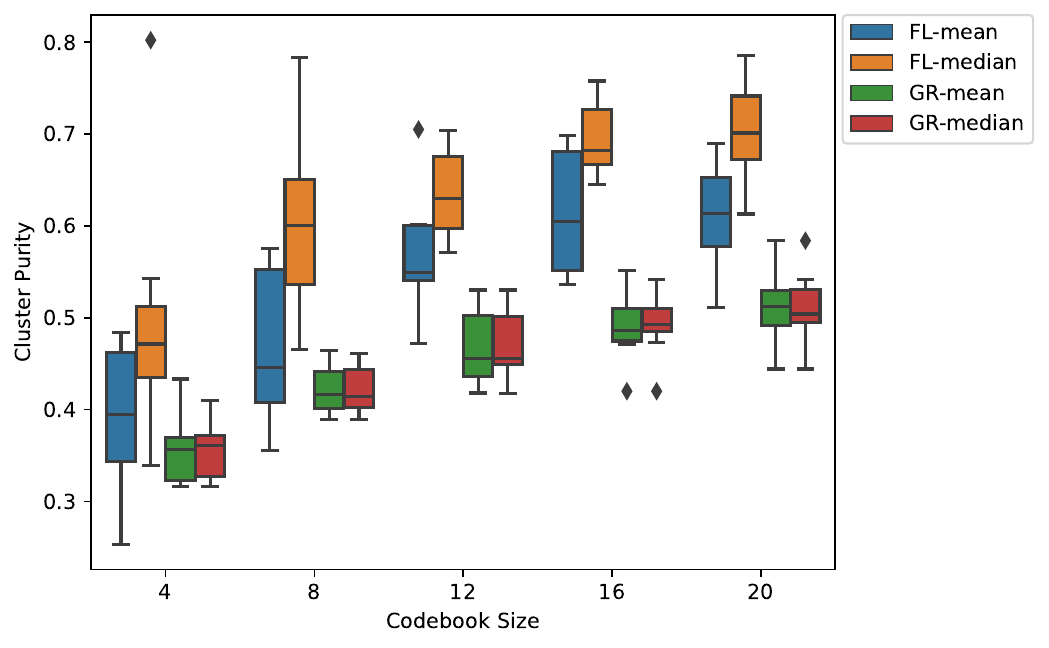}
	\caption{LBG cluster purities of YouTube videos with $10$ experiments with different numbers of centers, codebook sizes. The Grassmannian, \eg ``GR-'', boxes are results from LBG clustering using chordal distance and Grassmannian averages from~\cite{draper2014flag,mankovich2022flag}. The ``FL-" boxes are results from using the flag chordal distance and our flag-mean and -median.\vspace{-3mm}}
	\label{fig:lbgres}
\end{figure}

\subsection{Ablation Studies}
\paragraph{Robustness to initialization}
For Fig.~\ref{fig:init_ablation}, we fix a single-cluster dataset of $100$ points on $\flag(1,2,3;10)$ then run our IRLS algorithm for the flag-median and Stiefel RTR~\cite{absil2007trust,boumal2014manopt} for the flag-mean with initial points that are further and further away from the center of the dataset. Our dataset is computed the same way we compute synthetic datasets for the manuscript: compute a center, $[\![\mathbf{C}]\!] \in \flag(1,2,3;10)$, and then add noise to the center using the parameter $\delta$. For this experiment we use $\delta = .2$. Our initial point for our IRLS algorithm and RTR is computed as the first $3$ columns of the QR decomposition of $\mathbf{C}  + \Z \delta_{init}$ where $\Z \in \R^{10 \times 3}$ has entries sampled from $\mathcal{U}[-.5,.5)$. We call $\delta_{init}$ the noise added to the initial point and plot it on the $x$-axis of Fig.~\ref{fig:init_ablation}. The ``Error" is the chordal distance on $\flag(1,2,3;10)$ between the center and the algorithm output. ``Iterations" is the number of iterations of RTR for the flag-mean and IRLS for the flag-median until convergence. ``Cost" is the objective function values of the algorithm output. Our IRLS algorithm estimates the flag-median is further away from the center, $[\![\mathbf{C}]\!]$ than the flag-mean estimate. Also, the number of iterations of Stiefel RTR increases as we move the center further away from our dataset whereas our IRLS algorithm number of iterations remains constant. Finally, the cost value for the flag-median estimate is higher than the flag-mean estimate because the flag-mean estimate objective function likely contains squares of values less than $1$.

\begin{figure}[t]
    \includegraphics[width=\linewidth]{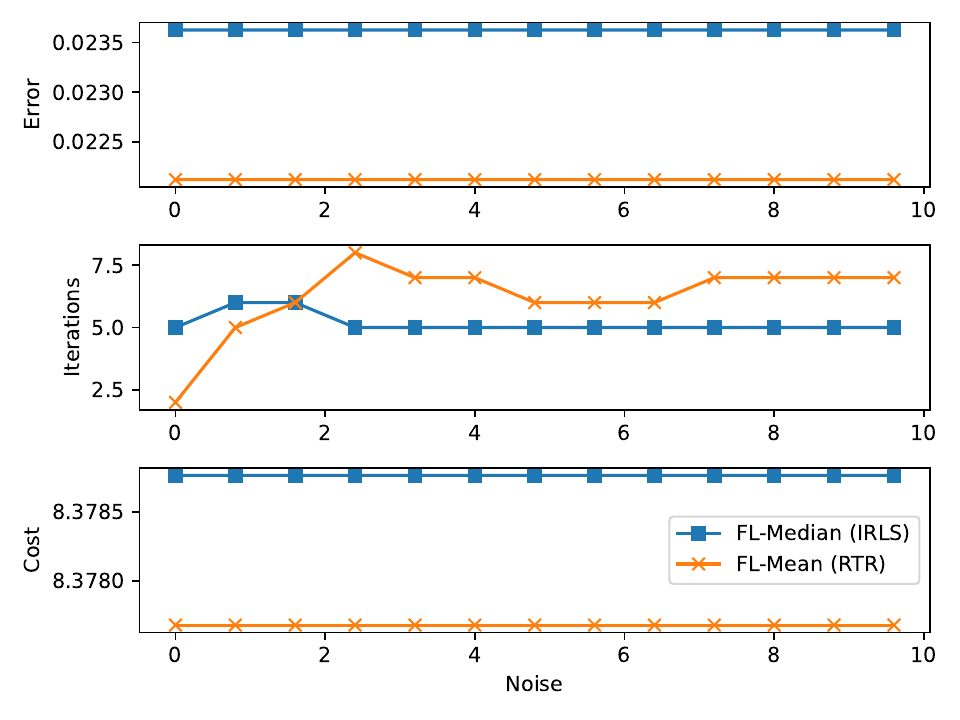}
    \caption{A plot of the robustness of our IRLS algorithm for the flag-median and Stiefel RTR for the flag-mean to initialization. For the median, we report the IRLS-iterations whereas for the mean, we report the RTR-iterations. Note that, even in large noise variances, both of the algorithms converge to a reasonable point regardless of initialization.}
    \label{fig:init_ablation}
\end{figure}

\paragraph{Computation time}

We conducted the further experiments with ambient dimension and ``dimension gap'' and plot the runtime of Alg. 1 (Fl-Mean) and Alg. 2 (Fl-Median) in Fig.~\ref{fig:times}. As shown, the runtime increases linearly with dimension and decreases linearly with dimension gap. The FL-Mean is less affected than the FL-Median when changing $d$ or $d-k$. For high $d$, the runtime for the FL-Median is unstable with high standard deviation and high changes in the mean runtime across small changes in $d$. In contrast, the FL-Mean is relatively stable in run-time to increasing $d$. When we vary dimension gap ($d-k$), there is a negligible standard deviation in runtime for the FL-Mean and FL-Median. The FL-Median algorithm is very slow for low $d-k$ and as fast as the FL-Mean for high $d-k$. The FL-Mean algorithm runtime is more stable to changes in dimension gap than the FL-Median.

\begin{figure}[t]
        \includegraphics[width=\columnwidth]{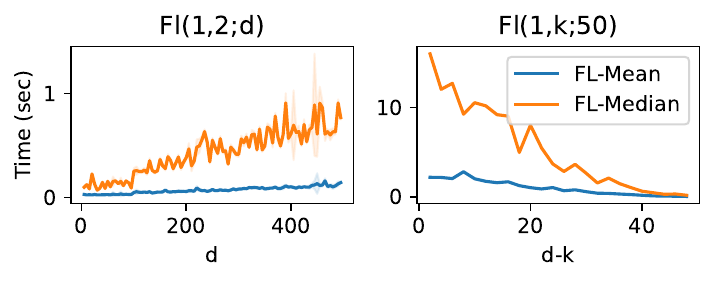}\vspace{-5mm}
	\caption{Time to compute the chordal flag-mean and -median of $10$ points over $20$ random trials. The shaded region is the standard deviation. We vary $d$ in $\flag(1,2;d)$ (left) and vary $d-k$ in $\flag(1,k;d=50)$ (right).\vspace{-4.5mm}}
	\label{fig:times}
\end{figure}


\subsection{Motion Averaging}

\paragraph{On error metrics}
We score the quality of our averages using the geodesic distance on the pose manifold $SE(3)$ (or equivalently the geodesic distance on dual quaternions):
\begin{equation}\label{eq:se3error}
    \epsilon(\mathbf{T}_1,\mathbf{T}_2) = {\frac{1}{\pi}\|\log(\mathbf{R_1}^\top\mathbf{R}_2)\|_2+\lambda_T\|\mathbf{t}_1-\mathbf{t}_1\|_2}
\end{equation}
where $(\mathbf{R}_i,\mathbf{t}_i)$ are extracted from $\mathbf{T}_i$ as rotational and translational components, respectively. $\lambda_T$ is a scene dependent strictly positive scaling factor. Note that, as discussed in the paper, this is very much related to the $\lambda$ used in motion contraction. $\log(\cdot):SO(3)\to \mathfrak{so}(3)$ denotes the \emph{logarithmic map} of the $SO(3)$-manifold. As such, this residual defined in Fig.~\ref{eq:se3error} is equivalent to:
\begin{equation}\label{eq:tracedist}
    \epsilon(\mathbf{T}_1,\mathbf{T}_2) = {\frac{1}{\pi}\arccos\left(\mathrm{tr} \frac{\mathbf{R}_1^\top\mathbf{R}_2 -1}{2} \right)+\lambda_s\|\mathbf{t}_1-\mathbf{t}_2\|_2}.\nonumber
\end{equation}

\paragraph{Single rotation averaging}
Single rotation averaging where a set of rotation matrices are averaged, is a special case of motion averaging where the translational components are set to zero. Due to the compactness of the manifold, additional $SO(3)$-specific averaging algorithms can be employed for the case of pure rotations. To compare our method against a larger class of well established, rotation-specific averaging algorithms we opt for zeroing the translational components, performing averages and reporting only the angular errors. Figs.~\ref{fig:rotavg1},\ref{fig:rotAvgRobust} present our results with increasing noise and increasing outliers respectively. For the case of outliers, we further include the recent robust methods of Rie \& Civera~\cite{lee2020robust}. \emph{Naive} refers to the Euclidean averages of rotation matrices (with a subsequent projection). 

\paragraph{Impact of $\lambda$}
As we have discussed in the paper, the scene-dependent scaling $\lambda$ is a hyper-parameter in our $SE(3)$-averaging. Note that, other distance metrics such as the ones dependent on 3D point distances exist~\cite{bregier2018defining}. These metrics exploit the action of 3D transformations on an auxiliary point set to measure distances in the 3D configuration of points. However, even those are somehow dependent upon a hyper-parameter such as the diameter of the point set or the point configurations. This is why we evaluate the impact of $\lambda$ in our averaging. In particular, we design multiple experiments to average $250$ random points on $SE(3)$ generated with an angular noise level of $0.075$. The radius of this scene is set to $1$, up to a translational noise level of $0.15$. This also means that the optimal $\lambda^\star$ (unknown during test) is $1$. We then vary $\lambda\in [0.002, 0.025, 0.1, 0.25, 0.5, 0.75, 1, 1.25, 1.5, 1.8, 2, 2.25,\\2.5]$ and average $250$ random points, over $50$ runs. In each run, the point sets differ randomly. We compute the errors using Eq.~\ref{eq:tracedist} with $\lambda=1$ and accumulate them over all runs. Fig.~\ref{fig:lambdaexp} plots the average errors per each $\lambda$. In this outlier-free regime, our flag-mean and flag-median are almost aligned when $\lambda=1$. Flag-median shows slight advantage over the mean for smaller values of $\lambda$.
\begin{figure}[t]
        \includegraphics[width=\columnwidth]{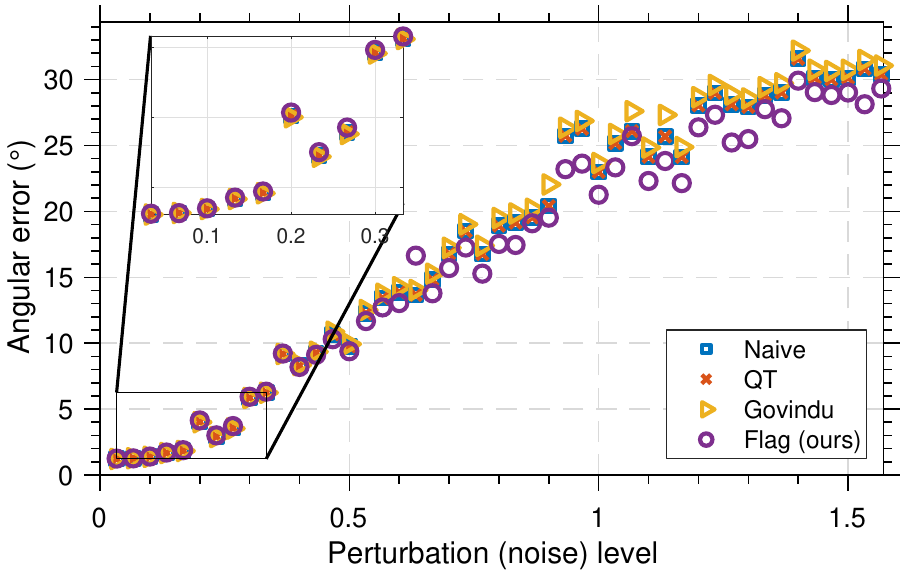}
	\caption{Single rotation averaging results for increasing levels of axial noise on synthetic, outlier-free data.}
	\label{fig:rotavg1}
\end{figure}
\begin{figure}[t]
        \includegraphics[width=\columnwidth]{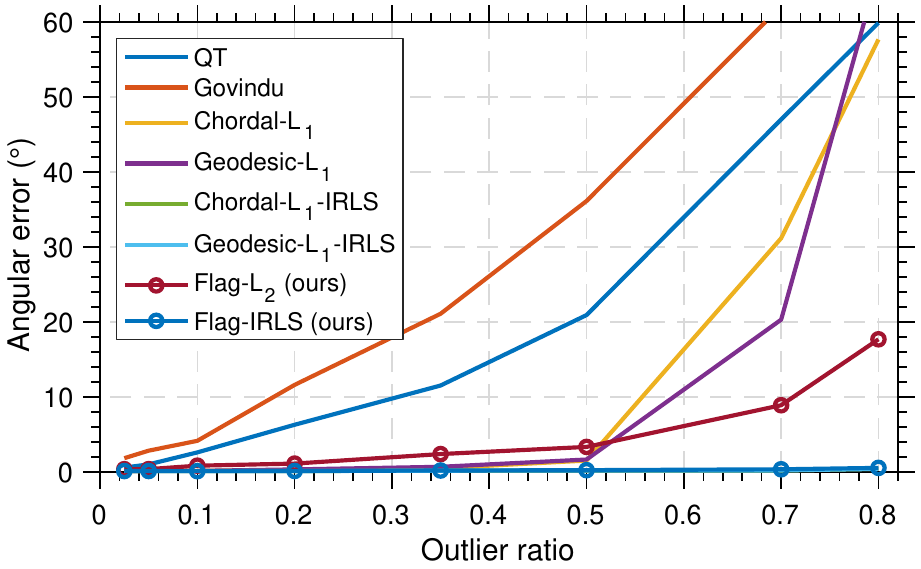}
	\caption{Single rotation averaging results for increasing levels of outliers on synthetic data with $<5^\circ$ of noise.}
	\label{fig:rotAvgRobust}
\end{figure}

\begin{figure}[ht]
    \includegraphics[width=\linewidth]{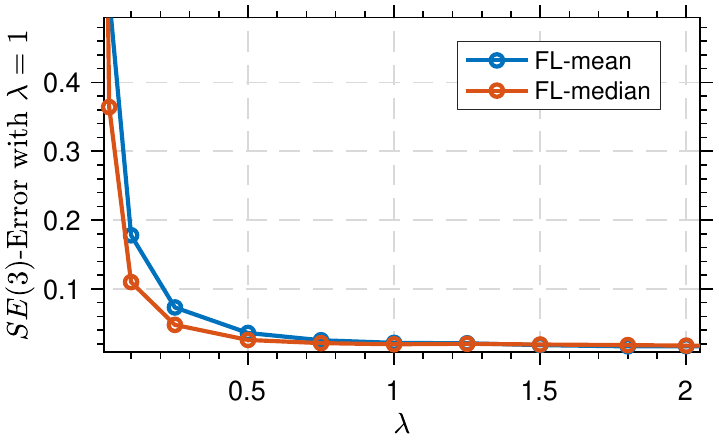}
    \caption{Impact of $\lambda$ our flag-mean and flag-median. During data generation we use $\lambda^\star=1$, whereas our algorithms use varying $\lambda$ as plotted in the $x$-axis. We then compare the resulting averages to the ground truth average and report the deviation. This is an outlier-free regime and as expected, median \& mean prototypes overlap when we are at the optimal value, $\lambda=1$. Though, we also see that our algorithms are not too sensitive to the exact choice of this parameter.\vspace{-3mm}}
    \label{fig:lambdaexp}
\end{figure}

\section{On Motion \& Rotation Averaging}
Motion averaging lies at the heart of structure from motion and 3D reconstruction in multi-view settings. Typically, the problem of recovering individual motions for a set of cameras when we are given a number of relative motion estimates between camera pairs is known as \emph{multiple motion averaging} or \emph{transformation / motion synchronization}~\cite{huang2019learning,arrigoni2016spectral,birdal2018bayesian,dellaert2020shonan,eriksson2019rotation,chatterjee2017robust,govindu2001combining,govindu2004lie}. This problem is foundational for 3D structure recovery. Synchronization algorithms usually solve multiple \emph{single averaging} sub-problems robustly~\cite{birdal2020synchronizing,li2020hybrid,hartley2011l1}, hence the name multiple motion averaging. These sub-problems involving the computation of a robust-barycenter of a set of points on $SE(3)$, are commonly known as \emph{robust single motion averaging} and is the focus of our paper. Although our method directly operates on the product manifold, it is a de-facto standard to decompose the problem into \emph{single translation averaging} and \emph{single rotation averaging}. The lattter is particularly well studied due to the interesting mathematical structure of the problem~\cite{hartley2013rotation,hartley2011l1,lee2020robust,govindu2004lie}. Nevertheless, our method is general enough to solve all of these variants, as we have experimented with in Figs.~\ref{fig:rotavg1},\ref{fig:rotAvgRobust}.

\end{document}